\documentclass[journal,twoside,web]{IEEEtran}
\usepackage[english]{babel}
\usepackage{graphicx} 
\graphicspath{{ {./images/} }}
\usepackage{epsfig} 
\usepackage{amsmath} 
\usepackage{amssymb}  
\usepackage{color}
\usepackage{subcaption}
\usepackage{caption}
\usepackage{authblk}
\usepackage{cite}
\usepackage{amsthm, amsmath}
\newtheorem{theorem}{Theorem}
\newtheorem{proposition}{Proposition}
\theoremstyle{definition}

\newtheorem{assumption}{Assumption}
\newtheorem{remark}{Remark}
\newtheorem{lemma}{Lemma}
\usepackage{algorithm}
\usepackage{algpseudocode}
\usepackage[T1]{fontenc}
\usepackage{bm}
 
\algrenewcommand\algorithmicindent{1.2em}

\usepackage{mathtools}      
\usepackage{bm}             
\usepackage{microtype}
\usepackage{enumitem}

\newcommand{\N}{\mathbb{N}}

\newcommand{\Ghat}{\boldsymbol{G}^{\mathrm{pred}}}
\newcommand{\Vhat}{\boldsymbol{V}^{\mathrm{pred}}}

\newcommand{\Jbar}{\bar{J}}
\newcommand{\Lmax}{\Lambda_{\max}}
\newcommand{\norm}[1]{\left\|#1\right\|}

\begin{document}

\title{SACK : Safe Active Continual Koopman Learning for Uncertain Systems with Contractive Guarantees
\author{Chandan Kumar Sah*, Rajpal Singh*, and Jishnu Keshavan}
\thanks{* These authors have contributed equally to the present work.}
\thanks{The authors are with the Department of Mechanical Engineering, Indian Institute of Science, Bangalore, Karnataka~560012, India 
    ({\tt\footnotesize{email: chandanks@iisc.ac.in, rajpalsingh@iisc.ac.in, kjishnu@iisc.ac.in}}).
    }
}

\maketitle
\IEEEpeerreviewmaketitle

\begin{abstract}
Koopman operator theory provides a powerful framework for representing nonlinear dynamics through a linear operator acting on lifted observables, enabling the use of linear control techniques for nonlinear systems. However, Koopman models are typically learned from data and often degrade in performance under model uncertainty and distributional shifts between training and deployment. Although several works have explored online adaptation to address this issue, many rely on neural network-based updates that introduce significant computational overhead and lack formal safety guarantees, limiting their suitability for real-time and safety-critical robotic applications. In this work, we propose SACK, a unified framework for continual adaptive Koopman learning that enables safe and efficient online refinement of learned models during task execution. A Koopman model is first learned offline and subsequently refined online through a contractive adaptation law, which provides theoretical convergence guarantees under distributional shifts and model uncertainty. To improve data efficiency and accelerate model refinement, the adaptation mechanism is integrated with an active learning strategy that drives the system to collect informative data while accomplishing task objectives. The resulting control problem is formulated as a nonconvex optimization problem incorporating both active learning objectives and safety constraints. We further derive theoretical bounds on model approximation error and show how these bounds can be incorporated within a robust Model Predictive Control (MPC) framework to provide formal safety guarantees. To reduce conservatism in practice, we also introduce a conformal prediction-based tightening mechanism that calibrates safety margins online from observed residuals. The proposed approach unifies learning, excitation, and safety within a single framework without sacrificing real-time feasibility. Extensive simulation and experimental studies across various robotic platforms demonstrate superior performance compared to state-of-the-art baselines in safety-critical environments.
\end{abstract}

\section{Introduction}
\label{sec:introduction}
Control of nonlinear dynamical systems is a central challenge in engineering, arising across robotic manipulation, aerial vehicles, legged locomotion, and a wide range of physical and biological systems~\cite{dixon2003nonlinear, rasmussen2011inference, wang2008filtering}. Although model-based control techniques for nonlinear systems are well developed, they tend to be system-specific and computationally demanding, especially under optimality or hard constraints~\cite{strogatz2018nonlinear}, whereas linear control offers scalable, computationally efficient tools with strong theoretical guarantees~\cite{ogata2010modern}. This disparity has long motivated efforts to construct globally valid linear representations of nonlinear systems, enabling linear analysis and control synthesis beyond the narrow operating regimes afforded by first-order Taylor linearization.

Koopman operator theory provides a principled framework for representing nonlinear dynamical systems through a linear operator acting on an infinite-dimensional space of observable functions~\cite{koopman1931hamiltonian}. This formulation yields a linear, though generally infinite-dimensional, representation of the underlying nonlinear dynamics~\cite{korda2018linear}. In practice, finite-dimensional approximations of the Koopman operator are learned from data using approaches such as Dynamic Mode Decomposition (DMD)~\cite{champion2019data}, Extended Dynamic Mode Decomposition (EDMD)~\cite{li2017extended}, and neural network-based methods~\cite{lusch2018deep}. These frameworks have subsequently been extended to controlled systems~\cite{brunton2021modern}. Within robotics, Koopman-based models have been successfully applied to a broad range of platforms, including industrial manipulators~\cite{sah2024real,singh2025generalized}, soft robotic systems~\cite{bruder2020data,bruder2025koopman}, aerial vehicles~\cite{zinage2021koopman,folkestad2022koopnet,manaa2024koopman}, and legged robots~\cite{krolicki2022modeling,yang2025koopman}. These studies demonstrate that Koopman representations enable the application of linear control and optimization techniques to complex nonlinear robotic systems while retaining strong empirical performance. More recently, Koopman-based prediction has been used in safety-critical control, where linear lifted dynamics enable tractable safety verification. In particular, Koopman models have been combined with control barrier functions (CBFs); for example, Koopman-based linear prediction is used in \cite{folkestad2020data} to avoid costly backup trajectory integration while accounting for model error.
 
Despite the success of Koopman-based control, learned models are usually trained under nominal operating conditions that rarely match the true environment, which fundamentally limits closed-loop performance. Distributional shifts, such as payload variations, aerodynamic disturbances, or contact-induced effects, frequently alter the underlying dynamics and render the offline-learned Koopman operator inaccurate. As a result, prediction errors accumulate during operation, degrading control performance and potentially invalidating safety guarantees established using the nominal model. These challenges motivate the development of adaptive Koopman frameworks that refine the learned model online during deployment.
 
Several recent works address online Koopman adaptation. The study in \cite{singh2025adaptive} augments an offline model with an auxiliary network that learns the residual mismatch in real time, and extends this approach to jointly optimize the lifting map and operator matrices online using soft target-network stabilization \cite{uchida2024model}. The study in \cite{li2024continual} incrementally expands the lifted-state dimension to provably reduce approximation error for high-dimensional legged robots, while a Bayesian meta-learning prior over Koopman operators is adopted in \cite{selimmetakoopman} for closed-form, optimization-free adaptation. Each, however, carries a limitation consequential in safety-critical deployment. The neural updates of~\cite{singh2025adaptive, uchida2024model} require iterative gradient steps at every control step, forcing a fundamental trade-off: the number of gradient steps must be small enough to meet latency requirements, yet sufficient to reduce model error meaningfully. Under rapid distributional shift, precisely the very conditions that motivate adaptation, this trade-off typically favors speed over accuracy, and neither approach provides a safety guarantee in the transient phase. The incremental scheme of~\cite{li2024continual} sidesteps convergence by growing the model, but scales the lifted-state MPC, adding to real-time computational burden. In \cite{selimmetakoopman}, closed-form updates are obtained only when the deployment distribution lies within the meta-training support, an assumption that is hard to verify and may fail on genuinely novel conditions. The lightweight pseudo-inverse update of~\cite{banday2025event} lacks convergence analysis, offering no assurance that the adapted model improves on the nominal one.
 
Another important limitation shared by these works concerns data quality during online adaptation. Closed-loop controllers typically drive the system toward smooth and repetitive trajectories, causing the collected regressors to become increasingly correlated over time. Consequently, the resulting data provides progressively less information about the true Koopman operator, while the directions most critical for accurate adaptation remain insufficiently excited. The EDMD-based active learning approach of~\cite{abraham2019active} is a notable exception. Still, it does not enforce state or input constraints, making it unsuitable for safety-critical settings where information-seeking excitation could lead to constraint violation. Moreover, its EDMD-based update repeatedly re-estimates the Koopman operator from locally collected trajectory windows. This makes the adapted model sensitive to the conditioning and representativeness of the recent data, and can lead to poor update quality under large distributional shifts or insufficiently informative samples. The absence of an offline-trained prior further increases the risk of unreliable behavior during the early stages of closed-loop operation.

Beyond Koopman-based methods, a parallel line of research studies the exploration-safety trade-off using alternative uncertainty representations. \cite{soloperto2020augmenting} augment nonlinear MPC with a tunable active-learning objective, using economic and multi-objective MPC to bound performance degradation, but without explicitly accounting for safety during online model updates. The study in \cite{baltussen2025dual} incorporates the dual effect by conditioning a Gaussian process (GP) posterior on the predicted control sequence and combining it with contingency-horizon robust MPC to maintain recursive feasibility. Trajectory-level exploration is considered in \cite{naveed2025formal}, executing informative trajectories only when they are certified to remain safe and within a predefined performance budget. At the other end of the spectrum, the study in \cite{lew2022safe} uses Bayesian meta-learned dynamics with reachability-based chance-constrained planning to provide high-probability safety over unknown horizons. In contrast, the study in \cite{prajapat2025safe} leverages GP confidence sets and reachable-returnable safe sets to obtain finite-time sample-complexity guarantees for safe exploration. Collectively, these works demonstrate that safe active learning is possible using GP and meta-learned models. However, they rely on computationally intensive posterior updates, gradient-based meta-adaptation, or separate candidate generation and safety verification procedures, rather than a single closed-form adaptive recursion with an explicit online error bound. This motivates the proposed contractive Gramian-based adaptation law, which enables CBF-tightened active-learning MPC within a unified optimization framework.

Thus, we propose Safe Active Continual Koopman Learning (SACK), a unified framework for continual model adaptation in constrained environments with explicit contraction and safety guarantees. The framework jointly integrates online learning, information-driven excitation, and safety-constrained control within a single optimization architecture. First, rather than relying on neural network-based updates, we derive an explicit \emph{contractive adaptation law} for the Koopman operator parameters. By formulating the parameter error dynamics as a contracting system governed by the windowed data Gramian, we obtain provable exponential convergence of the estimates under persistent excitation and an explicit ultimate bound. Second, because the convergence rate depends directly on the conditioning of the data Gramian, we couple the adaptation law with a \emph{D-optimal active learning objective}, the log-determinant of the predicted regressor Gramian, that promotes excitation simultaneously across all regressor directions. Unlike~\cite{abraham2019active}, this objective is embedded within a constrained receding-horizon optimization that enforces actuator limits, state feasibility, and collision avoidance at every step, ensuring that information-driven excitation never compromises safety. Third, we derive deterministic bounds that account for both model mismatch and the perturbation introduced by online parameter updates, and incorporate them into a CBF-constrained MPC to establish recursive feasibility and forward-invariance guarantees. To reduce conservatism, we further introduce a distribution-free conformal tightening calibrated from observed residuals that provides probabilistic safety guarantees.
 
The main contributions of this work are:
\begin{enumerate}
  \item \textbf{Contractive online Koopman adaptation.}
    We derive a closed-form recursive update law for the Koopman model with provably contracting error dynamics, yielding exponential convergence under persistent excitation and ultimate boundedness under time-varying, out-of-distribution dynamics.
 
  \item \textbf{Safety-constrained active-learning MPC.}
    We formulate a D-optimal information objective that maximizes excitation across all regressor directions, directly accelerating the contraction rate above, embedded in a nonconvex MPC that jointly enforces actuator limits, state constraints, and CBF conditions, so information-seeking exploration never compromises closed-loop safety.
 
  \item \textbf{End-to-end guarantees for the coupled learning-control loop.}
    We prove recursive feasibility of the MPC, forward invariance of the safe set, and ultimate boundedness of the lifted state under a disturbance bound that accounts for both model mismatch and online parameter updates. We further introduce a distribution-free conformal tightening that certifies per-step and finite-horizon safety coverage, reducing conservatism relative to worst-case bounds.
 
  \item \textbf{Validation across the sim-to-real spectrum.}
    We validate SACK through simulations of a 3R manipulator and a planar quadrotor under distributional shifts and disturbances, high-fidelity Gazebo simulations of a 7-DoF Franka Research 3 arm, and hardware experiments on a TurtleBot3 mobile robot. Across all platforms, SACK achieves improved tracking accuracy, sustained constraint satisfaction, and greater robustness to distributional shifts than existing adaptive and active-learning Koopman methods~\cite{singh2025adaptive,abraham2019active} and a non-Koopman baseline~\cite{baltussen2025dual}.
\end{enumerate}
 
The remainder of this paper is organized as follows. Section~\ref{sec:methodology} reviews the Koopman framework, introduces the adaptive model, and derives the contractive adaptation law and its convergence analysis. Section~\ref{sec:active_learning} presents the active-learning formulation and the constrained MPC optimization problem. Section~\ref{sec:anal_dist_bound} develops deterministic bounds and  establishes the corresponding safety guarantees. Section~\ref{sec:conf_pred} provides a probabilistic alternative to Section~\ref{sec:anal_dist_bound}. Section~\ref{sec:results} presents simulation and experimental results, and Section~\ref{sec:conclusion} presents the conclusions and future work.

\color{black}

\section{Adaptation for Koopman Models}\label{sec:methodology}

\subsection{Koopman Operator Theory Preliminaries}
\label{subsec:koopman_prelim}
Consider a nonlinear autonomous system
\begin{equation}
  \dot{\boldsymbol{x}} = f(\boldsymbol{x}),
  \label{eq:autonomous}
\end{equation}
where $\boldsymbol{x}\in\mathcal{X}\subset\mathbb{R}^{n_x}$ and $f:\mathcal{X}\rightarrow\mathbb{R}^{n_x}$ is Lipschitz
continuous, with associated discrete-time flow map $\boldsymbol{x}_{k+1}=S(\boldsymbol{x}_k)$, $S:\mathcal{X}\rightarrow\mathcal{X}$.
Koopman operator theory lifts the dynamics to a space of scalar-valued observable functions $\varphi:\mathcal{X}\rightarrow\mathbb{R}$,
on which the evolution is described by the linear Koopman operator $\mathcal{K}$,
\begin{equation}
  \mathcal{K}\circ\varphi(\boldsymbol{x}_k) = \varphi\circ S(\boldsymbol{x}_k) = \varphi(\boldsymbol{x}_{k+1}),
\end{equation}
where $\circ$ denotes function composition. The defining property of $\mathcal{K}$ is linearity: it evolves observables linearly even when $f$ is nonlinear. An exact representation, however, generally requires an infinite-dimensional observable space~\cite{korda2018linear,mauroy2020koopman}.
For practical use one therefore selects a finite dictionary $\{\phi_1,\ldots,\phi_p\}$, $\phi_j:\mathcal X\rightarrow\mathbb R$,
collected into the \emph{lifting map}
\begin{equation}
  \boldsymbol{z}_k := \psi(\boldsymbol{x}_k) = \big[\phi_1(\boldsymbol{x}_k),\ldots,\phi_p(\boldsymbol{x}_k)\big]^\top \in\mathbb{R}^p,
  \label{eq:lifting}
\end{equation}
and restricts $\mathcal{K}$ to $\mathrm{span}\{\phi_1,\ldots,\phi_p\}$. The resulting model is exact only if this subspace is $\mathcal{K}$-invariant, and otherwise incurs a truncation error quantified below.

The framework extends to controlled systems. Consider the control-affine system
\begin{equation}
  \dot{\boldsymbol{x}} = f_0(\boldsymbol{x}) + \sum_{i=1}^{n_u} f_i(\boldsymbol{x})\,u_i,
  \label{eq:control_affine}
\end{equation}
where $\boldsymbol{u}=[u_1,\ldots,u_{n_u}]^\top\in\mathbb{R}^{n_u}$, $f_0:\mathcal{X}\rightarrow\mathbb{R}^{n_x}$ is the drift
vector field, and $f_i:\mathcal{X}\rightarrow\mathbb{R}^{n_x}$ are the control vector fields, all Lipschitz continuous on
$\mathcal{X}$. Under the assumption stated in~\cite{singh2025adaptive,goswami2017global}, \eqref{eq:control_affine} admits
the discrete-time lifted representation
\begin{equation}
  \hat{\boldsymbol{z}}_{k+1|k} = \boldsymbol{A}\boldsymbol{z}_k + \boldsymbol{B}\boldsymbol{u}_k,
  \qquad
  \hat{\boldsymbol{x}}_{k+1|k} = \boldsymbol{C}\hat{\boldsymbol{z}}_{k+1|k},
  \label{eq:koopman_discrete}
\end{equation}
with $\boldsymbol{A}\in\mathbb{R}^{p\times p}$, $\boldsymbol{B}\in\mathbb{R}^{p\times n_u}$ the nominal model matrices and $\boldsymbol{C} \in\mathbb{R}^{n_x\times p}$ the reconstruction matrix recovering the state $\boldsymbol{x}$ from the lifted state. Here $\hat{\boldsymbol{z}}_{k+1|k}$ is the one-step-ahead prediction produced by the nominal model from the measured lifted state $\boldsymbol{z}_k$ and input $\boldsymbol{u}_k$. The conditioning ``$\,\cdot\,|k$'' denotes a quantity predicted from information available up to step $k$. The tuple $\{\psi(\cdot),\boldsymbol A,\boldsymbol B,\boldsymbol C\}$
is identified offline from nominal input-output data (Appendix~\ref{app:offline_learning}). Since \eqref{eq:koopman_discrete} is built on a finite dictionary and identified from nominal data, it is an approximation:
the true lifted dynamics satisfy
\begin{equation}
  \boldsymbol{z}_{k+1} = \boldsymbol{A}\boldsymbol{z}_k + \boldsymbol{B}\boldsymbol{u}_k + \boldsymbol{\epsilon}_k,
  \label{eq:truncation}
\end{equation}
where $\boldsymbol{\epsilon}_k$ collects two distinct contributions: the truncation error incurred by restriction $\mathcal{K}$ to a $p$-dimensional subspace, and the mismatch between the nominal matrices $(\boldsymbol A,\boldsymbol B)$ and the true Koopman dynamics. The first is fixed once the dictionary is chosen; the second is not. Under distributional shift, the deployed dynamics depart from the true dynamics, so the second
contribution grows during operation. This component is, however, expressible within the same lifted coordinates and can therefore be reduced online by refining $(\boldsymbol A,\boldsymbol B)$, which motivates the adaptive module developed next.

\subsection{Adaptive Koopman Module}
\label{subsec:adaptive_module}

Distributional shift is modelled by augmenting~\eqref{eq:control_affine} with unknown perturbations,
\begin{equation}
  \dot{\boldsymbol{x}} = f_0(\boldsymbol{x})+\tilde f_0(\boldsymbol{x})
  + \sum_{i=1}^{n_u}\big(f_i(\boldsymbol{x})+\tilde f_i(\boldsymbol{x})\big)u_i,
  \label{eq:perturbed_physical}
\end{equation}
where $\tilde f_0,\tilde f_i:\mathcal{X}\rightarrow\mathbb{R}^{n_x}$ arise from parameter variations, unmodeled dynamics, or
environmental disturbances, and are unknown at deployment. In principle such a shift perturbs both the dictionary
$\psi(\cdot)$ and the operator matrices $(\boldsymbol A,\boldsymbol B)$, and adapting both would be the ideal response.
Re-learning $\psi(\cdot)$ online, however, requires retraining the dictionary at every control step, which is incompatible with real-time operation.

\begin{assumption}[Fixed lifting map]
\label{as:fixed_lifting}
The perturbed dynamics~\eqref{eq:perturbed_physical} remain representable within the offline-learned observable subspace
$\mathrm{span}\{\phi_1,\ldots,\phi_p\}$. Consequently $\psi(\cdot)$ is held fixed during deployment, and adaptation is
performed only over $(\boldsymbol A,\boldsymbol B)$.
\end{assumption}

Assumption~\ref{as:fixed_lifting} is reasonable for the shifts of practical interest: payload changes, aerodynamic
disturbances, and actuator degradation alter how the state evolves rather than which functions of the state are needed to
describe it, so the dictionary learned under nominal conditions typically remains adequate while the operator acting on it
does not. This is consistent with the empirical findings of~\cite{singh2025adaptive}. Restricting adaptation to
$(\boldsymbol A,\boldsymbol B)$ additionally preserves a \emph{linear} parameter-estimation problem, admitting the
closed-form recursive update and convergence guarantees.

Under Assumption~\ref{as:fixed_lifting}, the perturbed system admits the lifted representation
\begin{equation}
  \boldsymbol z_{k+1} = \boldsymbol A_k^\ast\boldsymbol z_k + \boldsymbol B_k^\ast\boldsymbol u_k,
  \qquad
  \boldsymbol x_{k+1} = \boldsymbol C\boldsymbol z_{k+1},
  \label{eq:perturbed_lifted}
\end{equation}
where $\boldsymbol W_k^\ast := [\boldsymbol A_k^\ast,\,\boldsymbol B_k^\ast]\in\mathbb{R}^{p\times(p+n_u)}$ denotes the
true, generally time-varying, finite-dimensional lifted operator associated with the fixed dictionary
$\{\phi_1,\ldots,\phi_p\}$; it is the object the adaptation law estimates. Should the perturbed dynamics leave the learned
subspace, so that Assumption~\ref{as:fixed_lifting} fails, the residual $\boldsymbol\epsilon_k$ of~\eqref{eq:truncation}
cannot be removed by operator adaptation alone. In summary, $\psi(\cdot)$ and $\boldsymbol C$ are learned offline and held fixed, whereas
$(\boldsymbol A,\boldsymbol B)$ are adapted online. Defining the stacked regressor
$\boldsymbol v_k := [\boldsymbol z_k^\top,\,\boldsymbol u_k^\top]^\top\in\mathbb{R}^{p+n_u}$, \eqref{eq:perturbed_lifted}
reads compactly as $\boldsymbol z_{k+1}=\boldsymbol W_k^\ast\boldsymbol v_k$.

\begin{assumption}[Bounded operator drift]
\label{as:bounded_drift}
There exists $\nu \ge 0$ such that $    \|\boldsymbol{W}_{k+1}^{\ast}-\boldsymbol{W}_{k}^{\ast}\|_{F}
    \le \nu, \;\forall\, k \ge 0.$
\end{assumption}

Assumption~\ref{as:bounded_drift} is standard in adaptive estimation~\cite{singh2025adaptive} and requires only that the
true operator vary at a bounded rate. The adaptation law therefore tracks a drifting target, with achievable estimation
accuracy fundamentally limited by $\nu$.

Since $\boldsymbol W_k^\ast$ is unknown, we maintain an adaptive estimate
$\hat{\boldsymbol W}_k=[\hat{\boldsymbol A}_k,\,\hat{\boldsymbol B}_k]$, with estimation error
$\boldsymbol E_k := \boldsymbol W_k^\ast-\hat{\boldsymbol W}_k\in\mathbb{R}^{p\times(p+n_u)}$. The resulting one-step
prediction error is
\begin{equation}
  \boldsymbol\varepsilon_{k+1} := \boldsymbol z_{k+1}-\hat{\boldsymbol z}_{k+1|k}
  = \boldsymbol W_k^\ast\boldsymbol v_k - \hat{\boldsymbol W}_k\boldsymbol v_k
  = \boldsymbol E_k\boldsymbol v_k.
  \label{eq:prediction_error}
\end{equation}
The adaptation law below updates $\hat{\boldsymbol W}_k$ from closed-loop measurements to drive $\|\boldsymbol E_k\|_F$ into a bounded neighbourhood determined by $\nu$.
 
\begin{figure*}[ht!]
      \centering
      \includegraphics[width=0.9\textwidth]{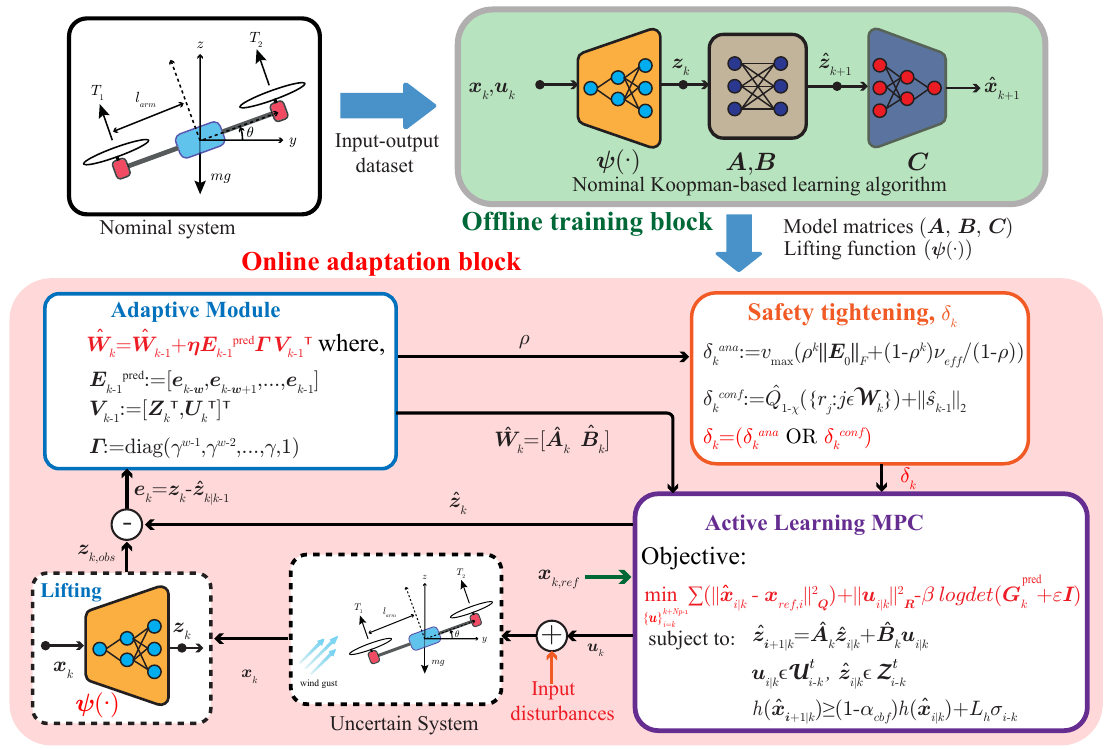}
    \caption{\textbf{Safe Active Continual Koopman (SACK) closed-loop architecture}. An offline learning module identifies the nominal Koopman model $(\boldsymbol{A},\boldsymbol{B},\boldsymbol{C})$ and the lifting function $\psi(\cdot)$ from input-output data. During deployment, prediction-error updates are used to adapt the Koopman operator online. The resulting uncertainty bounds are incorporated into an active-learning MPC~\eqref{eq:mpc} with tightened safety constraints, enabling safe and adaptive control under large distributional shift.}
        \label{fig:active_block}
\end{figure*}
 
\subsection{Online Adaptation Law}
\label{subsec:adaptation_law}
To improve conditioning and noise robustness, the update uses the $w(\ge 1)$ most recent regressor-observation pairs rather than a rank-one correction. Define the stacked matrices
\begin{align}
\bm{V}_k &:= [\bm{v}_{k-w+1},\dots,\bm{v}_k]\in\mathbb{R}^{(p+n_u)\times w}, \label{eq:Vk}\\
\bm{Z}^{+}_k &:= [\bm{z}_{k-w+2},\dots,\bm{z}_{k+1}]\in\mathbb{R}^{p\times w},
\label{eq:Zk}
\end{align}
so that each column of $\bm{Z}^{+}_k$ is the one-step-ahead observation of the corresponding column of $\bm{V}_k$. An exponential weighting matrix is introduced to discount the contribution of past observations: $\boldsymbol{\Gamma}
    := \mathrm{diag}\!\bigl(\gamma^{w-1},\;\gamma^{w-2},\;\ldots,\; \gamma,\;1\bigr) \in \mathbb{R}^{w\times w},$ where $0 < \gamma \leq 1$ is the forgetting factor. Smaller values of $\gamma$ place greater emphasis on recent observations, while $\gamma = 1$ corresponds to a uniformly weighted window. The weighted regressor Gramian is defined as $\boldsymbol{G}_k
    := \boldsymbol{V}_k \boldsymbol{\Gamma} \boldsymbol{V}_k^\top
     \in \mathbb{R}^{(p+n_u)\times(p+n_u)}.$
By construction, $\boldsymbol{G}_k$ is symmetric positive semidefinite, and becomes positive definite whenever the columns of $\boldsymbol{V}_k$ span $\mathbb{R}^{p+n_u}$, i.e., whenever the persistent excitation condition is satisfied over the current window. The windowed prediction error matrix is defined as $\boldsymbol{E}^{\mathrm{pred}}_k
    := \boldsymbol{Z}_k^+ - \hat{\boldsymbol{W}}_k\boldsymbol{V}_k
     \in \mathbb{R}^{p\times w}.$


\begin{assumption}[Bounded regressors]\label{as:bounded_regressors}
There exists $v_{\max} := \sup_{k\ge 0}\|\bm{v}_k\|_2 < \infty$.
\end{assumption}

Because the true operator drifts \emph{within} the window (Assumption~\ref{as:bounded_drift}), the columns of $\bm{Z}^{+}_k$ are generated by different operators. Indexing the columns of $\bm{V}_k$ by their sampling times $t_j := k-w+j$, $j = 1,\dots,w$, the $j$-th column of $\bm{Z}^{+}_k$ satisfies $\bm{z}_{t_j+1} = \bm{W}^{*}_{t_j}\bm{v}_{t_j}$, and hence the windowed residual admits the exact decomposition
\begin{equation}
\label{eq:Epred_decomp}
\bm{E}^{\mathrm{pred}}_k
  = \bm{E}_k\bm{V}_k + \bm{\Xi}_k,
\;
\bm{\Xi}_k := \bigl[\bm{D}_{k,1}\bm{v}_{t_1},\dots,\bm{D}_{k,w}\bm{v}_{t_w}\bigr],
\end{equation}
where $\bm{D}_{k,j} := \bm{W}^{*}_{t_j}-\bm{W}^{*}_k = -\sum_{i=t_j}^{k-1}\bm{\Delta}_i$ collects the intra-window operator drift, with $\bm{\Delta}_i := \bm{W}^{*}_{i+1}-\bm{W}^{*}_i$, and Assumption~\ref{as:bounded_drift} gives $\|\bm{D}_{k,j}\|_F \le (w-j)\,\nu$. For $w = 1$ the drift term vanishes identically, $\bm{\Xi}_k \equiv \bm{0}$, and \eqref{eq:Epred_decomp} reduces to the exact rank-one identity $\bm{E}^{\mathrm{pred}}_k = \bm{E}_k\bm{v}_k$.

The adaptive Koopman update rule is given by
\begin{equation}
  \label{eq:adaptation}
  \hat{\boldsymbol{W}}_{k+1}
    = \hat{\boldsymbol{W}}_k
      + \eta\,\boldsymbol{E}^{\mathrm{pred}}_k
        \boldsymbol{\Gamma}\boldsymbol{V}_k^\top,
  \qquad \eta > 0,
\end{equation}
where $\eta$ is a constant step size. Substituting \eqref{eq:Epred_decomp}, the update~\eqref{eq:adaptation} is equivalently expressed as
\begin{equation}
\label{eq:adap_law}
  \hat{\boldsymbol{W}}_{k+1}
    = \hat{\boldsymbol{W}}_k
      + \eta\,\boldsymbol{E}_k\boldsymbol{G}_k
      + \eta\,\bm{\Xi}_k\boldsymbol{\Gamma}\bm{V}_k^\top,
\end{equation}
whose first correction term projects the estimation error onto the directions spanned by the current regressor window $\boldsymbol{V}_k$, weighted by $\boldsymbol{\Gamma}$, while the second is a bias induced by intra-window operator drift and vanishes for $w=1$ or $\nu = 0$.

\begin{proposition}\label{prop:error_recursion}
Under \eqref{eq:adaptation}, the estimation error satisfies
\begin{equation}
\bm{E}_{k+1} = \bm{E}_k(\bm{I}-\eta\bm{G}_k)+\tilde{\bm{\Delta}}_k,
\;
\tilde{\bm{\Delta}}_k := \bm{\Delta}_k - \eta\,\bm{\Xi}_k\bm{\Gamma}\bm{V}_k^\top.
\label{eq:recursion}
\end{equation}
Moreover, under Assumptions~\ref{as:bounded_drift} and~\ref{as:bounded_regressors},
\begin{equation}
\label{eq:nu_eff}
\|\tilde{\bm{\Delta}}_k\|_F
\;\le\;
\nu_{\mathrm{eff}}
:= \nu\bigl(1 + \eta\, v_{\max}^2\, S_\gamma\bigr),
\;
S_\gamma := \sum_{i=0}^{w-1} i\,\gamma^{i}.
\end{equation}
\end{proposition}
\begin{proof}
Substituting \eqref{eq:adap_law} into $\bm{E}_{k+1}=\bm{W}^{*}_{k+1}-\hat{\bm{W}}_{k+1}$ yields
$\bm{E}_{k+1}=(\bm{W}^{*}_k-\hat{\bm{W}}_k)-\eta\bm{E}_k\bm{V}_k\bm{\Gamma}\bm{V}_k^{\top}-\eta\bm{\Xi}_k\bm{\Gamma}\bm{V}_k^\top+(\bm{W}^{*}_{k+1}-\bm{W}^{*}_k)=\bm{E}_k(\bm{I}-\eta\bm{G}_k)+\tilde{\bm{\Delta}}_k$, establishing \eqref{eq:recursion}. For \eqref{eq:nu_eff}, write $\bm{\Xi}_k\bm{\Gamma}\bm{V}_k^\top = \sum_{j=1}^{w}\gamma^{w-j}\,\bm{D}_{k,j}\bm{v}_{t_j}\bm{v}_{t_j}^\top$. Since $\|\bm{D}\bm{v}\bm{v}^\top\|_F = \|\bm{D}\bm{v}\|_2\|\bm{v}\|_2 \le \|\bm{D}\|_F\|\bm{v}\|_2^2$, the triangle inequality together with $\|\bm{D}_{k,j}\|_F \le (w-j)\nu$ and Assumption~\ref{as:bounded_regressors} gives
$\|\bm{\Xi}_k\bm{\Gamma}\bm{V}_k^\top\|_F \le \nu\, v_{\max}^2 \sum_{j=1}^{w}\gamma^{w-j}(w-j) = \nu\, v_{\max}^2\, S_\gamma$, and $\|\tilde{\bm{\Delta}}_k\|_F \le \|\bm{\Delta}_k\|_F + \eta\,\nu\, v_{\max}^2 S_\gamma \le \nu_{\mathrm{eff}}$.
\end{proof}

The recursion~\eqref{eq:recursion} decomposes the error evolution into a linear contraction term governed by the windowed Gramian $\boldsymbol{G}_k$ and an additive perturbation $\tilde{\bm{\Delta}}_k$ attributable to temporal variation of the true operator, comprising the one-step drift $\bm{\Delta}_k$ and the intra-window drift bias accumulated across the $w$ samples. When $\boldsymbol{G}_k \succ \boldsymbol{0}$ and the step size satisfies $\rho_k := \|\boldsymbol{I} - \eta\boldsymbol{G}_k\|_2 < 1$, the linear map $\boldsymbol{E}_k \mapsto \boldsymbol{E}_k(\boldsymbol{I} - \eta\boldsymbol{G}_k)$ is strictly contractive. The precise step-size condition ensuring uniform contraction and the resulting exponential convergence bound are established in Lemma~\ref{lem:contraction} and Theorem~\ref{thm:contraction} below.

\begin{lemma}\label{lem:contraction}
Consider \eqref{eq:recursion} with $\bm{G}_k\succeq 0$, and define the uniform eigenvalue upper bound
$\bar{\lambda} := c_\gamma v_{\max}^2 \;\ge\; \lambda_{\max}(\bm{G}_k)\ \forall k$, where $c_\gamma := \sum_{i=0}^{w-1}\gamma^{i}$, which holds under Assumption~\ref{as:bounded_regressors}. Then:
(i) if $0<\eta<2/\lambda_{\max}(\bm{G}_k)$, the update is non-expansive, $\|\bm{E}_{k+1}\|_F\leq\|\bm{E}_k\|_F+\|\tilde{\bm{\Delta}}_k\|_F$;
(ii) if additionally $\bm{G}_k\succ 0$, it is strictly contractive, $\|\bm{E}_{k+1}\|_F\leq\rho_k\|\bm{E}_k\|_F+\|\tilde{\bm{\Delta}}_k\|_F$ with $\rho_k:=\|\bm{I}-\eta\bm{G}_k\|_2<1$;
(iii) if $\lambda_{\min}(\bm{G}_k)\geq\underline{\lambda}>0$ for all $k$ and $0 < \eta < 2/\bar{\lambda}$, the contraction is uniform:
$\|\bm{E}_{k+1}\|_F\leq\rho\|\bm{E}_k\|_F+\|\tilde{\bm{\Delta}}_k\|_F$ for all $k$, with
$\rho := \max\bigl\{|1-\eta\underline{\lambda}|,\,|1-\eta\bar{\lambda}|\bigr\} < 1$ independent of $k$.
\end{lemma}
\begin{proof}
By submultiplicativity of the Frobenius norm, $\|\bm{E}_{k+1}\|_F\leq\|\bm{I}-\eta\bm{G}_k\|_2\|\bm{E}_k\|_F+\|\tilde{\bm{\Delta}}_k\|_F$. Since $\bm{G}_k$ is symmetric positive semidefinite, $\|\bm{I}-\eta\bm{G}_k\|_2=\max_i|1-\eta\lambda_i(\bm{G}_k)|$. Under $0<\eta<2/\lambda_{\max}(\boldsymbol G_k)$, $|1-\eta \lambda_i(\boldsymbol G_k)|\le 1$ for all $i$, proving (i). If $\boldsymbol G_k \succ 0$, then $\lambda_i(\boldsymbol G_k)>0$ for all $i$, hence $|1-\eta\lambda_i(\boldsymbol G_k)|<1$, establishing (ii). For (iii), note first that $\lambda_{\max}(\bm{G}_k) \le \|\bm{V}_k\bm{\Gamma}\bm{V}_k^\top\|_2 \le \sum_{j=1}^{w}\gamma^{w-j}\|\bm{v}_{t_j}\|_2^2 \le c_\gamma v_{\max}^2 = \bar{\lambda}$, so every eigenvalue of $\bm{G}_k$ lies in $[\underline{\lambda},\bar{\lambda}]$. Since $\lambda \mapsto |1-\eta\lambda|$ is convex, it attains its maximum over $[\underline{\lambda},\bar{\lambda}]$ at an endpoint, giving $\|\bm{I}-\eta\bm{G}_k\|_2 \le \max\{|1-\eta\underline{\lambda}|,|1-\eta\bar{\lambda}|\} = \rho$. The condition $0<\eta<2/\bar{\lambda}$ ensures both endpoint values are strictly less than one, hence $\rho<1$ independently of $k$.
\end{proof}

\begin{theorem}\label{thm:contraction}
Suppose (a) Assumptions~\ref{as:bounded_drift} and~\ref{as:bounded_regressors} hold; (b) $\lambda_{\min}(\bm{G}_k)\ge\underline{\lambda}>0$ for all $k$ and the step size satisfies $0<\eta<2/\bar{\lambda}$, so that $\|\boldsymbol{I} - \eta \boldsymbol{G}_k \|_2 \le \rho < 1$ with $\rho$ as in Lemma~\ref{lem:contraction}(iii). Then $\| \boldsymbol{E}_{k+1} \|_F \le \rho \| \boldsymbol{E}_k \|_F + \nu_{\mathrm{eff}}$, and hence
\begin{equation}
\|\bm{E}_k\|_F {\leq} \rho^{k}\|\bm{E}_0\|_F{+}\frac{1-\rho^{k}}{1{-}\rho}\,\nu_{\mathrm{eff}}, \,
\limsup_{k\to\infty}\|\bm{E}_k\|_F{\leq}\frac{\nu_{\mathrm{eff}}}{1{-}\rho},
\label{eq:bound}
\end{equation}
with $\nu_{\mathrm{eff}} = \nu(1+\eta v_{\max}^2 S_\gamma)$ from \eqref{eq:nu_eff}.
\end{theorem}
\begin{proof}
By Lemma~\ref{lem:contraction}(iii) and Proposition~\ref{prop:error_recursion}, $\| \boldsymbol{E}_{k+1} \|_F \le \rho \| \boldsymbol{E}_k \|_F + \nu_{\mathrm{eff}}$. Unrolling the recursion yields
\begin{align}
\| \boldsymbol{E}_k \|_F {\le} \rho^k \| \boldsymbol{E}_0 \|_F {+} \sum_{i=0}^{k{-}1} \rho^i \nu_{\mathrm{eff}} {=} \rho^k \| \boldsymbol{E}_0 \|_F {+} \frac{1 {-} \rho^k}{1 {-} \rho} \nu_{\mathrm{eff}}.
\end{align}
Taking the limit superior completes the proof.
\end{proof}

\begin{remark}\label{rem:window_tradeoff}
The drift inflation factor in $\nu_{\mathrm{eff}}$ is governed by $S_\gamma = \sum_{i=0}^{w-1} i\gamma^i$, the $\gamma$-weighted average age of the samples in the window. For $w=1$, $S_\gamma = 0$ and Theorem~\ref{thm:contraction} recovers the exact rank-one bound with $\nu_{\mathrm{eff}}=\nu$. For $\gamma = 1$, $S_\gamma = w(w-1)/2$ grows quadratically in the window length, whereas for $\gamma < 1$, $S_\gamma \le \gamma/(1-\gamma)^2$ uniformly in $w$. The forgetting factor therefore actively suppresses the drift-induced bias: larger windows improve the conditioning of $\bm{G}_k$ (raising $\underline{\lambda}$ and hence the contraction rate) at the cost of an $O(\eta\nu v_{\max}^2 S_\gamma)$ bias, and choosing $\gamma < 1$ caps this cost independently of $w$. This quantifies the selection of moderate window lengths with exponential forgetting. Under the step-size condition $\eta < 2/\bar{\lambda} = 2/(c_\gamma v_{\max}^2)$, the inflation factor additionally satisfies $\nu_{\mathrm{eff}} \le \nu(1 + 2S_\gamma/c_\gamma)$, where $S_\gamma/c_\gamma \le w-1$ is the mean sample age; the bound is thus at most a modest multiple of $\nu$ for the window lengths used in practice.
\end{remark}

Theorem~\ref{thm:contraction} provides an explicit bound on the learning-induced model mismatch, which can be interpreted as a bounded additive disturbance and subsequently absorbed by a robust MPC-like controller to ensure robust closed-loop stability. All downstream quantities consume the estimation-error bound only through $\bar{E}_k := \rho^k\|\bm{E}_0\|_F + (1-\rho^k)\,\nu_{\mathrm{eff}}/(1-\rho)$.

\section{Active Learning Algorithm}\label{sec:active_learning}

The recursion $\bm{E}_{k+1}=\bm{E}_k(\bm{I}-\eta\bm{G}_k)+\bm{\tilde{\Delta}}_k$ of Theorem~\ref{thm:contraction} contracts the parameter error only in directions excited within the sliding window. Closed-loop tracking, however, drives the system toward smooth, repetitive trajectories, so the regressors $\bm{v}_k$ become correlated, yielding an ill-conditioned or low-rank $\bm{G}_k$, and consequently stalling or reversing online Koopman adaptation precisely when it is needed. This section addresses this data-quality problem by coupling the adaptation law with an active learning strategy that explicitly optimizes the informativeness of the closed-loop trajectory.


\subsection{D-Optimal Active Learning Objective}\label{sec:doptimal}
By Lemma~\ref{lem:contraction}, for a step size $\eta$ the contraction factor satisfies $\rho_k=\|\boldsymbol I-\eta\boldsymbol G_k\|_2=\max_i|1-\eta\lambda_i(\boldsymbol G_k)|<1$ if and only if $\lambda_{\min}(\boldsymbol G_k)>0$, and the uniform factor of Lemma~\ref{lem:contraction}(iii) is $\rho=\max\{|1-\eta\underline\lambda|,\,|1-\eta\bar\lambda|\}$. This factor is non-increasing in $\underline\lambda=\inf_k\lambda_{\min}(\boldsymbol G_k)$, and strictly decreasing while the $\underline\lambda$ branch is active, i.e., while $|1-\eta\underline\lambda|\ge|1-\eta\bar\lambda|$. Once the upper branch dominates, further excitation of the weakest direction no longer improves the guaranteed rate, and $\rho$ is limited by $\bar\lambda$ and the step size. Improving the conditioning of $\boldsymbol G_k$ is therefore the primary lever on the convergence rate, i.e., when the regressor window is poorly conditioned. This motivates augmenting the task objective with an information-seeking term that targets $\lambda_{\min}(\boldsymbol G_k)$ along the closed-loop trajectory.

To this end, we maximize the D-optimality criterion by augmenting the task cost with the log-determinant of the predicted Gramian $\Ghat_k$. Concretely, define the stacked predicted regressor matrix
\begin{align}
  \label{eq:Vhat}
  \Vhat_k
    &= \bigl[{\boldsymbol{\hat{v}}}_{k|k},\;
             \ldots,\;
             {\boldsymbol{\hat{v}}}_{k+N_p-1|k}\bigr]
    \in\mathbb{R}^{(p+n_u)\times N_p},\\
  \label{eq:Ghat}
  \Ghat_k
    &= \Vhat_k \Vhat_k{}^{\!\top}
    \in\mathbb{R}^{(p+n_u)\times(p+n_u)},
\end{align}
constructed from \emph{future} rollout regressors. The information metric is
\begin{equation}
  \label{eq:jinfo}
  J_{\mathrm{info}}
    := \log\det (\Ghat_k + \varepsilon\boldsymbol{I} ),
  \qquad \varepsilon > 0,
\end{equation}
where $\varepsilon\boldsymbol{I}$ ensures positive definiteness under rank deficiency early in adaptation. With exploration-exploitation weight $\beta>0$, the combined objective is
\begin{align}
  \label{eq:combined_obj}
  J &= J_{\mathrm{task}} - \beta\,J_{\mathrm{info}},\; \text{where}\\
  \label{eq:jtask}
  J_{\mathrm{task}}
    &:= \sum_{i=k}^{k+N_p-1}
         \| \boldsymbol{C} \boldsymbol{\hat{z}}_{i|k} - \boldsymbol{x}_{\mathrm{ref},i}\|^2_{\boldsymbol{Q}}
         + \|\boldsymbol{u}_{i|k}\|^2_{\boldsymbol{R}}.
\end{align}
The choice of log-det is motivated by the identity $\log\det(\Ghat_k + \varepsilon\boldsymbol{I})  = \sum_i \log(\lambda_i(\Ghat_k) + \varepsilon)$. Since $\log(\cdot+\varepsilon)$ is concave and steepest near zero, the criterion penalizes small eigenvalues most heavily. It thus promotes excitation across \emph{all} eigendirections of $\Ghat_k$, discouraging rank deficiency in any single direction. This is preferable to simpler heuristics such as maximizing input energy, which may over-excite already-excited directions while leaving others unimproved. The excitation actually guaranteed by this objective is quantified in Appendix~\ref{app:active_learn_better} Theorem~\ref{thm:gramian_lb}, and the influence of the exploration weight on the optimizer is characterized in Theorem~\ref{thm:perturbation}, which establishes that the active-learning solution varies smoothly from the nominal MPC solution while increasing the predicted information content for sufficiently small $\beta$.
 
\subsection{Constrained Active-Learning MPC}\label{sec:mpc_formulation}
The constrained optimization problem solved at each step $k$ is
\begin{subequations}\label{eq:mpc}
\begin{align}
\min_{\boldsymbol{U}}\ & \sum_{i=k}^{k+N_p-1}\!\Big( \big\|\boldsymbol{C} \boldsymbol{\hat{z}}_{i|k} - \boldsymbol{x}_{\mathrm{ref},i}\big\|_{\boldsymbol{Q}}^2 + \big\|\boldsymbol{u}_{i|k}\big\|_{\boldsymbol{R}}^2 \Big) \notag\\
& - \beta\,\log\det\!\Big( \Vhat_k(\boldsymbol{U})\,\Vhat_k(\boldsymbol{U})^{\!\top} + \varepsilon \boldsymbol{I} \Big) \label{eq:mpc_obj}\\
\text{s.t.}\ & \boldsymbol{\hat{z}}_{i+1|k} {=} \hat{\boldsymbol{A}}_k\, \boldsymbol{\hat{z}}_{i|k} {+} \hat{\boldsymbol{B}}_k\, \boldsymbol{u}_{i|k}, \ i=k,\dots,k{+}N_p{-}1, \label{eq:mpc_dyn}\\
& \boldsymbol{\hat{z}}_{k|k} = \boldsymbol{z}_{k}, \label{eq:mpc_init}\\
& \boldsymbol{u}_{i|k} \in \mathcal{U}^{\,t}_{i-k},\; \boldsymbol{\hat{z}}_{i|k} \in \mathcal{Z}^{\,t}_{i-k}, \ i=k,\dots,k{+}N_p{-}1, \label{eq:mpc_sets}\\
& \boldsymbol{\hat{z}}_{k+N_p|k}\in\mathcal{Z}_f, \label{eq:mpc_terminal}\\
& h\big(\boldsymbol{C} \boldsymbol{\hat{z}}_{i+1|k}\big) \ge (1-\alpha_{\mathrm{cbf}})\,h\big(\boldsymbol{C} \boldsymbol{\hat{z}}_{i|k}\big) + L_{h} \sigma_{i-k}, \notag\\
& \hspace{2.2cm} i=k,\dots,k{+}N_p{-}1, \label{eq:mpc_cbf}
\end{align}
\end{subequations}

where $\boldsymbol{U} := \{\boldsymbol{{u}}_{i|k}\}_{i=k}^{k+N_p-1}$ is the decision vector, $\boldsymbol{Q} \succeq \boldsymbol{0}$, $\boldsymbol{R} \succ \boldsymbol{0}$, $h(\cdot)$ is the CBF with decay rate $\alpha_{\mathrm{cbf}} \in (0,1]$, and $\mathcal Z^{\,t}_j$, $\mathcal U^{\,t}_j$, $\mathcal Z_f$, and the stage margins $\sigma_j$ are the tightened constraint sets, terminal set, and CBF margins. The tightening are precomputed offline from worst-case constants, so the online problem retains the structure of a nominal MPC. Both $\Vhat_k(\boldsymbol{U})$ and $\Ghat_k(\boldsymbol{U})$ depend on $\boldsymbol{U}$ through the rolled-out lifted states, which is the source of the nonconvexity in~\eqref{eq:mpc_obj}.

Problem~\eqref{eq:mpc} is solved in real time by sequential quadratic programming (SQP). Starting from a warm-started nominal input sequence $\{ \boldsymbol{{u}}_{i|k}^{(0)} \}$, SQP iteratively linearizes the CBF inequalities and the quadratic terminal constraint~\eqref{eq:mpc_terminal} around the current rollout, yielding affine constraints. The log-det exploration term is handled through its first-order expansion; its gradient admits the closed form $\nabla_{\Vhat}\log\det(\Vhat\Vhat{}^{\top}+\varepsilon\boldsymbol{I})=2(\Vhat\Vhat{}^{\top}+\varepsilon\boldsymbol{I}){}^{-1}\Vhat$, so each subproblem is a standard QP.
 
The following theorem provides a verifiable excitation certificate: it lower-bounds the minimum eigenvalue of the predicted Gramian for \emph{any} feasible input sequence whose composite cost does not exceed that of a known informative trajectory.

\begin{theorem}
\label{thm:gramian_lb}
Let $N_p\in\N$ be the prediction horizon, $p,n_u\in\N$ the lifted-state and input dimensions, respectively, $n \coloneqq p+n_u$, and let $\mathcal{F} \coloneqq \{\boldsymbol{U}\in\mathbb{R}^{n_uN_p}: \eqref{eq:mpc_dyn}\text{-}\eqref{eq:mpc_cbf}\text{ hold}\}$ denote the feasible set. Consider the composite objective
\begin{equation}
  J(\boldsymbol{U}) \;=\; J_{\mathrm{task}}(\boldsymbol{U}) \;-\; \beta\,J_{\mathrm{info}}(\boldsymbol{U}), \qquad \beta>0,
  \label{eq:obj}
\end{equation}
\begin{equation}
  J_{\mathrm{info}}(\boldsymbol{U}) \;\coloneqq\;
  \log\det\!\bigl(\Ghat_k(\boldsymbol{U}) + \varepsilon \boldsymbol{I}_n\bigr),
  \qquad \varepsilon > 0,
  \label{eq:info}
\end{equation}
with ${\Ghat_k}(\boldsymbol{U}) \coloneqq \Vhat_k(\boldsymbol{U})\Vhat_k(\boldsymbol{U})^{\!\top}\succeq0$ obtained by rolling out dynamics~\eqref{eq:mpc_dyn} under $\boldsymbol{U}$. Suppose:
\begin{enumerate}[label=(A\arabic*),leftmargin=2.6em,itemsep=1pt,topsep=2pt]
  \item \label{A1} \textbf{Bounded predicted regressors:} there exists $\hat{v}_{\max}<\infty$ such that $\norm{\hat{\boldsymbol{v}}_{i|k}}_2 \le \hat{v}_{\max}$ for all $i \in \{k,\ldots,k+N_p-1\}$ and all $\boldsymbol{U}\in\mathcal{F}$.
  \item \label{A2} \textbf{Bounded task cost:} $\exists$ $\bar{J}$ such that $0 \le J_{\mathrm{task}}(\boldsymbol{U}) \le \bar{J}<\infty$ on $\mathcal{F}$.
  \item \label{A3} \textbf{Informative feasible trajectory:} $\exists$ $\boldsymbol{U}^{\mathrm{good}}\in\mathcal{F}$ and $\lambda_0>0$ with, $\lambda_{\min}\!\bigl(\Ghat_k (\boldsymbol{U}^{\mathrm{good}})\bigr) \;\ge\; \lambda_0.$

  \item \label{A4} \textbf{Sufficient exploration weight:}
    \begin{equation}
      \frac{\bar{J}}{\beta}
      \;<\;
      n\log(\lambda_0 + \varepsilon)
      - (n-1)\log(\Lambda_{\max} + \varepsilon)
      - \log\varepsilon,
      \label{eq:A4}
    \end{equation}
    where $\Lambda_{\max} \coloneqq N_p \hat{v}_{\max}^2$.
\end{enumerate}
Then every $\boldsymbol{U}\in\mathcal{F}$ satisfying the cost-comparison test
\begin{equation}
  J(\boldsymbol{U}) \;\le\; J(\boldsymbol{U}^{\mathrm{good}})
  \label{eq:cert_test}
\end{equation}
satisfies
\begin{equation}
  \lambda_{\min}\!\bigl(\Ghat_k (\boldsymbol{U})\bigr) \;{\ge}\; \lambda^* {\coloneqq}
  \frac{(\lambda_0+\varepsilon)^{n}}{(\Lambda_{\max}+\varepsilon)^{n-1}}e^{-\bar{J}/\beta}
  {-} \varepsilon {>} 0.
  \label{eq:lambdastar}
\end{equation}
In particular, every global minimizer $\boldsymbol{U}^*$ of $J$ over $\mathcal{F}$ satisfies \eqref{eq:cert_test}, and hence \eqref{eq:lambdastar}.
\end{theorem}

\begin{proof}
We write $\lambda_j(\cdot)$ for the $j$-th eigenvalue of a symmetric matrix $(j = 1,\ldots,n)$.
By~\ref{A1}, $\lambda_j\!(\Ghat_k (\boldsymbol{U}))
\;\le\; \|{\Vhat_k(\boldsymbol{U})}\|_F^2
\;\le\; N_p\,\hat{v}_{\max}^2
\;=\; \Lambda_{\max}, \forall\,j$ for all $\boldsymbol{U}\in\mathcal{F}$. Hence, for any $\boldsymbol{U}$ with $\lambda_{\min}(\Ghat_k(\boldsymbol{U}))\le\delta>0$, monotonicity of $\log(\cdot+\varepsilon)$ gives
\begin{equation}
  J_{\mathrm{info}}(\boldsymbol{U})
  \;\le\;
  \phi(\delta) \;\coloneqq\;
  \log(\delta+\varepsilon)
  +
  (n-1)\log(\Lambda_{\max}+\varepsilon),
  \label{eq:Ibad}
\end{equation}
and therefore $J(\boldsymbol{U})\ge-\beta\phi(\delta)$ by~\ref{A2}. In contrast, by~\ref{A3}, $\lambda_j(\Ghat_k (\boldsymbol{U}^{\mathrm{good}})) \ge \lambda_0$ for all $j$, so $J_{\mathrm{info}}(\boldsymbol{U}^{\mathrm{good}})\ge n\log(\lambda_0+\varepsilon)\eqqcolon C_{\mathrm{good}}$, hence $J(\boldsymbol{U}^{\mathrm{good}})\le\bar{J}-\beta C_{\mathrm{good}}$. Any $\boldsymbol{U}$ with $\lambda_{\min}(\Ghat_k (\boldsymbol{U}))\le\delta$ therefore violates the test~\eqref{eq:cert_test} whenever $-\beta\phi(\delta)>\bar{J}-\beta C_{\mathrm{good}}$, i.e., whenever
\begin{equation}
  \delta
  \;<\;
  \exp\!\Bigl(C_{\mathrm{good}}
    - \frac{\bar{J}}{\beta}
    - (n-1)\log(\Lambda_{\max}+\varepsilon)\Bigr)
  - \varepsilon
  \;=\;
  \lambda^*,
  \label{eq:delta-ineq}
\end{equation}
where substituting $C_{\mathrm{good}}$ recovers the closed form in~\eqref{eq:lambdastar}, and \ref{A4} is precisely the requirement $\lambda^*>0$. In contrast, any $\boldsymbol{U}\in\mathcal{F}$ satisfying~\eqref{eq:cert_test} must have $\lambda_{\min}(\Ghat_k (\boldsymbol{U}))\ge\lambda^*$. The final claim follows since a global minimizer satisfies $J(\boldsymbol{U}^*)\le J(\boldsymbol{U})$ for all $\boldsymbol{U}\in\mathcal{F}$, in particular for $\boldsymbol{U}^{\mathrm{good}}$.
\end{proof}

Theorem~\ref{thm:gramian_lb} establishes that any feasible input sequence whose composite cost matches that of a known informative trajectory carries a certified excitation level $\lambda_{\min}(\Ghat_k)\ge\lambda^*>0$. The key idea is that a trajectory with $\lambda_{\min}(\Ghat_k)<\lambda^*$ incurs an information deficit that, under~\ref{A4}, outweighs any achievable reduction in task cost, and therefore cannot cost less than $\boldsymbol{U}^{\mathrm{good}}$. Consequently, the optimization cannot sacrifice all excitation to improve tracking performance. The theorem certifies excitation at the optimization level through the predicted Gramian. Establishing persistent excitation of the realized receding-horizon closed-loop trajectory remains an interesting direction for future work

\begin{remark}
\label{rem:A4}
Condition~\ref{A4} is exactly the positivity condition $\lambda^*>0$ for the bound~\eqref{eq:lambdastar}. It requires the exploration weight $\beta$ to be sufficiently large relative to the worst-case task cost $\bar{J}$ so that the information deficit of a poorly conditioned trajectory cannot be offset by its task-cost advantage. If~\ref{A4} is violated, the task objective may dominate the optimization, and no positive lower bound on the minimum eigenvalue of the predicted Gramian is guaranteed by this argument.
\end{remark}

\begin{remark}
\label{rem:params}
From~\eqref{eq:lambdastar}, $\lambda^*$ satisfies the following monotonicity properties:
(i)~$\lambda^*$ increases with $\lambda_0$: a more informative comparison trajectory yields a tighter guarantee; (ii)~$\lambda^*$ increases with $\beta$: a larger exploration weight strengthens the information-seeking incentive; (iii)~$\lambda^*$ decreases with $\Jbar$: a tighter task cost bound reduces conservatism; (iv)~$\lambda^*$ decreases with $\Lmax (= N_p \hat{v}_{\max}^2)$: larger regressor spread weakens the guarantee because the contrast between well- and ill-conditioned trajectories diminishes. In all cases $\lambda^* < \lambda_0$, since
$e^{-\Jbar/\beta} < 1$ and $(\lambda_0{+} \varepsilon)/(\Lmax{+}\varepsilon) \le 1$. The gap $\lambda_0 - \lambda^*$ captures the degradation attributable to a nonzero task cost. We note that for large $n$ the ratio $(\lambda_0+\varepsilon)^n/(\Lambda_{\max}+\varepsilon)^{n-1}$ renders $\lambda^*$ quantitatively conservative certifying that excitation cannot collapse entirely.
\end{remark}
\section{Deterministic Safety Guarantees}
\label{sec:anal_dist_bound}
This section quantifies the uncertainty introduced by model mismatch and online adaptation, and incorporates it into the active-learning MPC through safety tightening. We first derive a deterministic disturbance bound, followed by safety guarantees for the closed loop implementation.

\subsection{Disturbance and Model-Update Bounds}\label{sec:composite_bound}
Since the update law~\eqref{eq:adaptation} modifies $\hat{\boldsymbol W}_k$ at every step, the true lifted dynamics satisfy $\boldsymbol z_{k+1}=\hat{\boldsymbol W}_k\boldsymbol v_k+\boldsymbol d_k$, where $\boldsymbol d_k := (\boldsymbol W_k^*-\hat{\boldsymbol W}_k)\boldsymbol v_k = \boldsymbol E_k\boldsymbol v_k$ is the one-step prediction error incurred at time $k$ under the current estimate $\hat{\boldsymbol W}_k$. From Theorem~\ref{thm:contraction},
\begin{equation}
\label{eq:Ebar}
\|\boldsymbol E_k\|_F\le\bar E_k := \rho^k\|\boldsymbol E_0\|_F+\frac{1-\rho^k}{1-\rho}\,\nu_{\mathrm{eff}},
\end{equation}
with $\nu_{\mathrm{eff}}$ from~\eqref{eq:nu_eff}, so that, under Assumption~\ref{as:bounded_regressors},
\begin{equation}
\label{eq:delta_ana}
\|\boldsymbol d_k\|_2 \le
\delta_k^{\mathrm{ana}}
:=
v_{\max}\bar E_k.
\end{equation}
The analytical bound $\delta_k^{\mathrm{ana}}$ captures the model-mismatch disturbance entering the true state transition at step $k$.

Recursive feasibility additionally requires bounding the change in the estimated Koopman operator between consecutive iterations. For a fixed regressor $\boldsymbol v_k$, this prediction drift is given by $(\hat{\boldsymbol W}_{k+1}-\hat{\boldsymbol W}_k)\boldsymbol v_k$. By~\eqref{eq:adap_law}, the operator increment comprises the error-projection term $\eta\boldsymbol E_k\boldsymbol G_k$ and the intra-window drift term $\eta\boldsymbol\Xi_k\boldsymbol\Gamma\boldsymbol V_k^\top$. Using $\|\boldsymbol G_k\|_2\le\bar\lambda = c_\gamma v_{\max}^2$ (Lemma~\ref{lem:contraction}), the bound $\|\boldsymbol\Xi_k\boldsymbol\Gamma\boldsymbol V_k^\top\|_F \le \nu\, v_{\max}^2 S_\gamma$ established in Proposition~\ref{prop:error_recursion}, and submultiplicativity of the Frobenius norm,
\begin{equation}
\label{eq:mu_bound}
\|\hat{\boldsymbol W}_{k+1}
-
\hat{\boldsymbol W}_k\|_F
\le
\eta\,\bigl(c_\gamma \bar E_k + S_\gamma\,\nu\bigr)\, v_{\max}^2
=: \mu_k.
\end{equation}
Since $\|(\hat{\boldsymbol W}_{k+1}-\hat{\boldsymbol W}_k)\boldsymbol v_k\|_2\le\|\hat{\boldsymbol W}_{k+1}-\hat{\boldsymbol W}_k\|_F\|\boldsymbol v_k\|_2$, the drift in the prediction for a fixed regressor is bounded by $\mu_k v_{\max}$.

Both $\delta_k^{\mathrm{ana}}$ and $\mu_k$ are deterministic functions of $\rho$, $\nu$, $\|\boldsymbol E_0\|_F$, $v_{\max}$, and $k$. Of these, $\rho$ and $v_{\max}$ are determined by design quantities and the compact constraint sets, whereas $\nu$ and $\|\boldsymbol E_0\|_F$ involve the unknown true operator and must be replaced in implementation by conservative upper estimates, obtained, e.g., from the magnitude of the physical perturbations considered and the offline validation error of the nominal model, respectively. It is to be noted that with such estimates, both bounds are evaluable online. They become conservative as $\rho\to1$, precisely the correlated-regressor regime induced by tracking control, motivating the data-driven tightening developed in Section~\ref{sec:tightening}.

\subsection{Recursive Feasibility, Forward Invariance, and Boundedness}
\label{sec:safety}

This subsection establishes recursive feasibility of the MPC problem~\eqref{eq:mpc}, forward invariance of the safe set $\mathcal{S}$, and ultimate boundedness of the closed-loop prediction error, using the deterministic bounds $\delta_k^{\mathrm{ana}}$ and $\mu_k$ of Section~\ref{sec:composite_bound}. Throughout, we use the worst-case constants
\begin{eqnarray}
\label{eq:sup_constants}
\bar E := \max\Bigl\{\|\boldsymbol E_0\|_F,\ \tfrac{\nu_{\mathrm{eff}}}{1-\rho}\Bigr\},
\quad
\bar\delta^{\mathrm{ana}} := v_{\max}\bar E, \\
\bar\mu := \eta\bigl(c_\gamma \bar E + S_\gamma\nu\bigr)v_{\max}^2,
\end{eqnarray}
which dominate $\bar E_k$, $\delta_k^{\mathrm{ana}}$, and $\mu_k$ for all $k$ by the monotonicity of $\bar E_k$ in~\eqref{eq:Ebar}.

\begin{assumption}
\label{as:operator_set}
There exists a known compact convex set
$\mathcal{W}^*\subset\mathbb{R}^{p\times(p+n_u)}$
such that the true time-varying Koopman operator satisfies
$\boldsymbol W_k^*\in\mathcal{W}^*$ for all $k\ge0$.
\end{assumption}

Assumption~\ref{as:operator_set} states that, although the true Koopman operator may evolve over time because of distributional shift, model uncertainty, or slowly varying system parameters,
it remains confined to a known bounded uncertainty set. The compactness of $\mathcal{W}^*$ guarantees that all admissible operators are uniformly bounded, while convexity permits the use of a common Lyapunov certificate and robust control arguments over the entire family of operators. Since Theorem~\ref{thm:contraction} yields the uniform estimation error bound $\|\boldsymbol E_k\|_F\le\bar E$ for all $k$, the adaptive estimates satisfy
\begin{equation}
\label{eq:operator_ball}
\hat{\boldsymbol W}_k\in\mathcal{W} := \mathcal{W}^*\oplus\mathcal{B}_F(\bar E)
\qquad \forall k\ge0,
\end{equation}
where $\mathcal{B}_F(\bar E)$ denotes the Frobenius-norm ball of radius $\bar E$. Consequently, all subsequent stability,
recursive-feasibility, and safety guarantees need only be established uniformly over the fixed compact set
$\mathcal W$, rather than for the unknown trajectory
$\{\boldsymbol W_k^*\}_{k\ge0}$.

\begin{assumption}[Robust quadratic stabilizability]
\label{as:stab}
There exist $\boldsymbol{K}_{\mathrm{ctrl}} \in \mathbb{R}^{n_u\times p}$, $\boldsymbol P$ with $\underline{p}\boldsymbol I\preceq\boldsymbol P\preceq\bar p\boldsymbol I$, and $\tilde\alpha\in(0,1)$ such that, for every $[\boldsymbol A,\boldsymbol B]\in\mathcal{W}$,
\begin{equation}
\label{eq:cqlf}
(\boldsymbol A+\boldsymbol B\boldsymbol K_{\mathrm{ctrl}})^{\!\top}\boldsymbol P\,(\boldsymbol A+\boldsymbol B\boldsymbol K_{\mathrm{ctrl}})\;\preceq\;\tilde\alpha^2\boldsymbol P.
\end{equation}
\end{assumption}

Since $\boldsymbol P$ is a common Lyapunov matrix over $\mathcal{W}$, it certifies decay of the \emph{time-varying} transition products: for any $\{\hat{\boldsymbol W}_i\}\subset\mathcal{W}$ and $\boldsymbol A_{\mathrm{cl},i}:=\hat{\boldsymbol A}_i+\hat{\boldsymbol B}_i\boldsymbol K_{\mathrm{ctrl}}$,
\begin{equation}
\label{eq:tv_decay}
\bigl\|\boldsymbol A_{\mathrm{cl},k+j-1}\cdots\boldsymbol A_{\mathrm{cl},k}\bigr\|_2
\;{\le}\; c\,\tilde\alpha^{\,j},
\;
c{:=}\sqrt{\bar p/\underline p},\ \forall j,k{\ge}0.
\end{equation}
Given $(c,\tilde\alpha)$, define the horizon-uniform margin
\begin{equation}
\bar\delta^+
:=
c\,\bar\delta^{\mathrm{ana}}
+
\frac{c\,\bar\mu\, v_{\max}}{1-\tilde\alpha},
\qquad
\bar\delta_{x}^+:= \|\boldsymbol C\|_2\, \bar\delta^+,
\label{eq:horizon_uniform_margin}
\end{equation}
accounting for model mismatch and adaptation-induced prediction drift compounded over the horizon. The tightened sets $\{\mathcal{Z}^t_i,\mathcal{U}^t,\mathcal{Z}^t_f\}$ entering~\eqref{eq:mpc} are constructed from $\bar\delta^+$ and the offset sequence of Lemma~\ref{lem:candidate}.

\begin{assumption}[Constraint sets and robust terminal ingredients]
\label{as:terminal}
With $\mathcal{Z}:=\{\boldsymbol{z}:\boldsymbol{C}\boldsymbol{z}\in\mathcal{X}\}$, the sets $\mathcal{Z}$ and $\mathcal{U}$ are compact and convex. There exists a terminal set $\mathcal{Z}_f=\{\boldsymbol z:\boldsymbol z^\top\boldsymbol P\boldsymbol z\le\tau\}\subseteq\mathcal{Z}$ such that, for every $[\boldsymbol A,\boldsymbol B]\in\mathcal{W}$ and $\boldsymbol z\in\mathcal{Z}_f$: (i)~$\boldsymbol C\boldsymbol z\in\mathcal{X}$, $\boldsymbol K_{\mathrm{ctrl}}\boldsymbol z\in\mathcal{U}$; and (ii)~$(\boldsymbol A+\boldsymbol B\boldsymbol K_{\mathrm{ctrl}})\boldsymbol z+\boldsymbol d\in\mathcal{Z}_f$ for every $\|\boldsymbol d\|_2\le\bar\delta^+$.
\end{assumption}

Both conditions are offline-verifiable: (i) is an ellipsoid-in-polytope containment, and (ii) holds whenever $\tau\ge\bar p\,(\bar\delta^+)^2/(1-\tilde\alpha)^2$. The terminal cost decrease $V_f=\boldsymbol z^\top\boldsymbol P\boldsymbol z$, $V_f((\boldsymbol A+\boldsymbol B\boldsymbol K_{\mathrm{ctrl}})\boldsymbol z)\le\tilde\alpha^2V_f(\boldsymbol z)$, follows from~\eqref{eq:cqlf} and need not be assumed. The candidate terminal input in the feasibility proof is $\kappa_f(\boldsymbol z):=\boldsymbol K_{\mathrm{ctrl}}\boldsymbol z$; the MPC does not impose this law online, and $\boldsymbol U$ remains a free decision variable.

\begin{assumption}[Barrier function regularity]
\label{as:barrier}
There exists $L_h>0$ such that
$h(\boldsymbol{x}+\boldsymbol{e}) \geq h(\boldsymbol{x}) - L_h\|\boldsymbol{e}\|_2$
for all $\boldsymbol{x},\boldsymbol{e}\in\mathbb{R}^{n_x}$.
\end{assumption}
Assumption~\ref{as:barrier} holds globally for smooth $h$ with bounded gradient, and locally for the circular barrier $h(\boldsymbol{x}) = \|\boldsymbol{x} - \boldsymbol{c}\|^2_2 - d^2_{\mathrm{safe}}$ via the local constant~\eqref{eq:local_Lh}. By construction, the MPC~\eqref{eq:mpc} enforces the tightened barrier constraint
\begin{equation}
  \label{eq:cbf_tight}
  h\!\left(\boldsymbol{\hat{x}}_{k+1|k}\right)
    \;\geq\;
    (1-\alpha_{\mathrm{cbf}})\,h(\boldsymbol{x}_k)
    + L_h\,\delta_{x,k},
\end{equation}
for all $k\ge0$, with $\alpha_{\mathrm{cbf}}\in(0,1]$ and $\delta_{x,k}$ from~\eqref{eq:delta_implemented}. Theorem~\ref{thm:joint}(ii) applies for $\delta_{x,k}=\delta^{\mathrm{ana}}_{x,k}$; the conformal tightening is covered by Theorem~\ref{thm:conformal}.

\begin{assumption}[Initial feasibility]
\label{as:initial}
The MPC problem~\eqref{eq:mpc}, posed with $\{\mathcal{Z}^t_i,\mathcal{U}^t,\mathcal{Z}^t_f\}$, is feasible at $k=0$; $\boldsymbol{x}_0\in\mathcal{S}:=\{\boldsymbol{x}\in\mathcal{X}:h(\boldsymbol{x})\ge0\}$ and $\boldsymbol{z}_0=\psi(\boldsymbol{x}_0)\in\mathcal{Z}$.
\end{assumption}

\begin{lemma}[Robust invariant error set]\label{lem:rpi}
Let $\boldsymbol P$, $\tilde\alpha$, $\underline p$, $\bar p$ be as in Assumption~\ref{as:stab}, and let $\bar\delta^{\mathrm{ana}}$ be the worst-case disturbance bound from~\eqref{eq:sup_constants}. Define
\[
\rho^{\mathrm{tube}} := \frac{\sqrt{\bar p}\,\bar\delta^{\mathrm{ana}}}{1-\tilde\alpha},
\qquad
\mathcal E := \{\boldsymbol{e} : \boldsymbol{e}^\top \boldsymbol{P} \boldsymbol{e} \le (\rho^{\mathrm{tube}})^2\}.
\]
Then, for any error dynamics $\boldsymbol e_{k+1}=\boldsymbol A_{\mathrm{cl},k}\boldsymbol e_k+\boldsymbol d_k$ with $\hat{\boldsymbol W}_k\in\mathcal{W}$ and $\|\boldsymbol{d}_k\|_2\le\bar\delta^{\mathrm{ana}}$, the set $\mathcal E$ is robustly positively invariant: $\boldsymbol e_k\in\mathcal E \Rightarrow \boldsymbol e_{k+1}\in\mathcal E$. Moreover $\boldsymbol{0}\in\mathcal E$ and
$\mathcal E\subseteq\mathcal B(\bar e)$ with
$\bar e:=\rho^{\mathrm{tube}}/\sqrt{\underline p}=c\,\bar\delta^{\mathrm{ana}}/(1-\tilde\alpha)$, a fixed Euclidean ball.
\end{lemma}
\begin{proof}
By~\eqref{eq:cqlf}, $\|\boldsymbol A_{\mathrm{cl},k}\boldsymbol e\|_{\boldsymbol P}\le\tilde\alpha\|\boldsymbol e\|_{\boldsymbol P}$ for every $\hat{\boldsymbol W}_k\in\mathcal{W}$, and $\boldsymbol P\preceq\bar p\boldsymbol I$ gives $\|\boldsymbol d_k\|_{\boldsymbol P}\le\sqrt{\bar p}\,\bar\delta^{\mathrm{ana}}$. Hence, for $\boldsymbol e\in\mathcal E$,
\[
\|\boldsymbol A_{\mathrm{cl},k}\boldsymbol e+\boldsymbol d_k\|_{\boldsymbol P}
\le
\tilde\alpha\rho^{\mathrm{tube}}
+
\sqrt{\bar p}\,\bar\delta^{\mathrm{ana}}
=
\rho^{\mathrm{tube}},
\]
by the definition of $\rho^{\mathrm{tube}}$, proving invariance. The inclusion follows from $\underline p\|\boldsymbol e\|_2^2\le\boldsymbol e^\top\boldsymbol P\boldsymbol e\le(\rho^{\mathrm{tube}})^2$ and $c=\sqrt{\bar p/\underline p}$.
\end{proof}

\begin{remark}
\label{rem:rpi_role}
Lemma~\ref{lem:rpi} characterizes the reachable set of any prediction-error process driven by the bounded disturbance $\boldsymbol d_k$ under the ancillary gain $\boldsymbol K_{\mathrm{ctrl}}$: since $\boldsymbol 0\in\mathcal E$, every such process initialized at zero remains in $\mathcal E$ for all time. Because the common matrix $\boldsymbol P$ of Assumption~\ref{as:stab} is valid uniformly over $\mathcal{W}$, a single fixed ellipsoid suffices despite the time-varying adapted matrices, and its Euclidean over-approximation $\mathcal B(\bar e)$ provides the fixed uncertainty set used for constraint tightening in Lemma~\ref{lem:candidate} and Theorem~\ref{thm:joint}. The analysis thus relies on Lyapunov-based invariance offline, while the online MPC requires only simple Euclidean tightening.
\end{remark}



Since the predicted trajectories of~\eqref{eq:mpc} are confined to $\mathcal Z\times\mathcal U$, we henceforth take $v_{\max}:=\max\{\|[\boldsymbol z^\top,\boldsymbol u^\top]^\top\|_2:\boldsymbol z\in\mathcal Z,\ \boldsymbol u\in\mathcal U\}$, which is finite by compactness (Assumption~\ref{as:terminal}) and bounds realized and predicted regressors alike; in particular, Assumption~\ref{as:bounded_regressors} holds with this constant, and it also serves as $\hat v_{\max}$ in Theorem~\ref{thm:gramian_lb}.

\begin{lemma}[Shifted-candidate feasibility]
\label{lem:candidate}
Let the MPC at time $k$ be feasible with optimal input sequence
$\boldsymbol U_k^\ast=\{\hat{\boldsymbol u}_{k|k}^\ast,\ldots,\hat{\boldsymbol u}_{k+N_p-1|k}^\ast\}$ and nominal prediction $\{\hat{\boldsymbol z}_{k+i|k}^\ast\}_{i=0}^{N_p}$, and let the first input $\boldsymbol u_k=\hat{\boldsymbol u}_{k|k}^\ast$ be applied. The measured re-initialization~\eqref{eq:mpc_init} then gives
\begin{equation}
\label{eq:zeta0}
\hat{\boldsymbol z}_{k+1|k+1}=\boldsymbol z_{k+1}
=\hat{\boldsymbol z}_{k+1|k}^\ast+\boldsymbol d_k,
\qquad \|\boldsymbol d_k\|_2\le\bar\delta^{\mathrm{ana}}.
\end{equation}
Define the candidate sequence for time $k+1$ by
\begin{eqnarray}
\label{eq:candidate}
\hat{\boldsymbol u}_{k+1+m|k+1}
:=\hat{\boldsymbol u}_{k+1+m|k}^{\ast}+\boldsymbol K_{\mathrm{ctrl}}\boldsymbol\zeta_m,
\\
\hat{\boldsymbol z}_{k+1+m|k+1}
:=\hat{\boldsymbol z}_{k+1+m|k}^{\ast}+\boldsymbol\zeta_m,
\end{eqnarray}
for $m=0,\ldots,N_p-2$, where $\boldsymbol\zeta_m$ is generated by
\begin{eqnarray}
\label{eq:zeta_rec}
\boldsymbol\zeta_0:=\boldsymbol d_k,
\qquad
\boldsymbol\zeta_{m+1}=\boldsymbol A_{\mathrm{cl},k+1}\boldsymbol\zeta_m+\boldsymbol\eta_m,
\\
\boldsymbol\eta_m {:=} (\hat{\boldsymbol A}_{k+1} {-} \hat{\boldsymbol A}_k) \hat{\boldsymbol z}_{k+1+m|k}^{\ast}
{+} (\hat{\boldsymbol B}_{k+1}-\hat{\boldsymbol B}_k)\hat{\boldsymbol u}_{k+1+m|k}^{\ast}.
\end{eqnarray}

Then:
\begin{enumerate}[label=(\roman*),itemsep=1pt,topsep=2pt]
\item the candidate~\eqref{eq:candidate} satisfies the prediction dynamics~\eqref{eq:mpc_dyn} under the updated model $(\hat{\boldsymbol A}_{k+1},\hat{\boldsymbol B}_{k+1})$ and the initialization~\eqref{eq:zeta0};
\item $\|\boldsymbol\eta_m\|_2\le\bar\mu\, v_{\max}$ and
\begin{eqnarray}
\label{eq:bm}
\|\boldsymbol\zeta_m\|_2\;\le\; b_m:=c\,\tilde\alpha^{\,m}\bar\delta^{\mathrm{ana}}
+\frac{c\,\bar\mu\, v_{\max}\bigl(1-\tilde\alpha^{\,m}\bigr)}{1-\tilde\alpha},\nonumber
\\ m=0,\ldots,N_p-1;
\end{eqnarray}
\item with $\varpi_j:=\sum_{\ell=0}^{j-1}b_\ell$, $\varpi_0:=0$, and $\varpi_j^{\mathrm{tot}}:=\bar e+\varpi_j$ ($\bar e$ from Lemma~\ref{lem:rpi}), the tightened sets
\[
\mathcal Z_i^{\,t}:=\mathcal Z\ominus\mathcal B(\varpi_i^{\mathrm{tot}}),
\qquad
\mathcal U_i^{\,t}:=\mathcal U\ominus\mathcal B\bigl(\|\boldsymbol K_{\mathrm{ctrl}}\|_2\,\varpi_i^{\mathrm{tot}}\bigr),
\]
satisfy: if $\hat{\boldsymbol z}_{k+1+m|k}^{\ast}\in\mathcal Z_{m+1}^{\,t}$ and $\hat{\boldsymbol u}_{k+1+m|k}^{\ast}\in\mathcal U_{m+1}^{\,t}$, then
$\hat{\boldsymbol z}_{k+1+m|k+1}\in\mathcal Z_m^{\,t}$ and $\hat{\boldsymbol u}_{k+1+m|k+1}\in\mathcal U_m^{\,t}$, for $m=0,\ldots,N_p-2$.
\end{enumerate}
\end{lemma}

\begin{proof}
(i) Under $(\hat{\boldsymbol A}_{k+1},\hat{\boldsymbol B}_{k+1})$, the candidate evolves as
$\hat{\boldsymbol A}_{k+1}(\hat{\boldsymbol z}^{\ast}_{k+1+m|k}+\boldsymbol\zeta_m)
+\hat{\boldsymbol B}_{k+1}(\hat{\boldsymbol u}^{\ast}_{k+1+m|k}+\boldsymbol K_{\mathrm{ctrl}}\boldsymbol\zeta_m)
=\hat{\boldsymbol A}_{k+1}\hat{\boldsymbol z}^{\ast}_{k+1+m|k}
+\hat{\boldsymbol B}_{k+1}\hat{\boldsymbol u}^{\ast}_{k+1+m|k}
+\boldsymbol A_{\mathrm{cl},k+1}\boldsymbol\zeta_m$.
Subtracting the time-$k$ prediction
$\hat{\boldsymbol z}^{\ast}_{k+2+m|k}=\hat{\boldsymbol A}_k\hat{\boldsymbol z}^{\ast}_{k+1+m|k}+\hat{\boldsymbol B}_k\hat{\boldsymbol u}^{\ast}_{k+1+m|k}$
shows the candidate state at stage $m+1$ equals $\hat{\boldsymbol z}^{\ast}_{k+2+m|k}+\boldsymbol\zeta_{m+1}$ with $\boldsymbol\zeta_{m+1}$ as in~\eqref{eq:zeta_rec}, i.e., the candidate is dynamically consistent; the base case is~\eqref{eq:zeta0} with $\boldsymbol\zeta_0=\boldsymbol d_k$.

(ii) By~\eqref{eq:mu_bound}, $\|[\hat{\boldsymbol A}_{k+1}-\hat{\boldsymbol A}_k,\hat{\boldsymbol B}_{k+1}-\hat{\boldsymbol B}_k]\|_F\le\mu_k\le\bar\mu$, and the stacked nominal regressor lies in $\mathcal Z\times\mathcal U$, so $\|\boldsymbol\eta_m\|_2\le\bar\mu\, v_{\max}$. The model is fixed over the horizon at time $k+1$, so $\boldsymbol A_{\mathrm{cl},k+1}$ is constant in the recursion, and~\eqref{eq:tv_decay} applies with the constant sequence: unrolling~\eqref{eq:zeta_rec},
$$
\boldsymbol\zeta_m=\boldsymbol A_{\mathrm{cl},k+1}^{\,m}\boldsymbol d_k
+\sum_{\ell=0}^{m-1}\boldsymbol A_{\mathrm{cl},k+1}^{\,m-1-\ell}\boldsymbol\eta_\ell,
$$
$$
\|\boldsymbol\zeta_m\|_2\le c\tilde\alpha^{\,m}\bar\delta^{\mathrm{ana}}
+\sum_{\ell=0}^{m-1}c\tilde\alpha^{\,m-1-\ell}\bar\mu v_{\max}=b_m.
$$

(iii) By construction $\varpi^{\mathrm{tot}}_{m+1}=\varpi^{\mathrm{tot}}_m+b_m$, hence $\mathcal B(b_m)\oplus\mathcal B(\varpi^{\mathrm{tot}}_m)\subseteq\mathcal B(\varpi^{\mathrm{tot}}_{m+1})$, and by monotonicity of the Pontryagin difference,
$\hat{\boldsymbol z}_{k+1+m|k+1}=\hat{\boldsymbol z}^{\ast}_{k+1+m|k}+\boldsymbol\zeta_m
\in\mathcal Z^{\,t}_{m+1}\oplus\mathcal B(b_m)\subseteq\mathcal Z^{\,t}_m$.
The input claim is identical with $\|\boldsymbol K_{\mathrm{ctrl}}\boldsymbol\zeta_m\|_2\le\|\boldsymbol K_{\mathrm{ctrl}}\|_2 b_m$.
\end{proof}

\begin{remark}
\label{rem:candidate_role}
The candidate correction absorbs the two perturbations that re-solving introduces: the realized disturbance $\boldsymbol d_k$, which enters through the measured re-initialization~\eqref{eq:zeta0} and decays along the horizon at rate $\tilde\alpha$, and the model update $(\hat{\boldsymbol W}_{k+1}-\hat{\boldsymbol W}_k)$, which accumulates through $\boldsymbol\eta_m$ toward the steady offset $c\bar\mu v_{\max}/(1-\tilde\alpha)$. The feedback gain $\boldsymbol K_{\mathrm{ctrl}}$ appears only in the \emph{candidate} construction; the implemented input remains $\boldsymbol u_k=\hat{\boldsymbol u}^{\ast}_{k|k}$, consistent with the re-initialization~\eqref{eq:mpc_init}. The base margin $\bar e$ in $\varpi^{\mathrm{tot}}_j$ additionally guarantees $\mathcal Z^{\,t}_0\subseteq\mathcal Z\ominus\mathcal B(\bar e)$, so the realized state, which coincides with the stage-$0$ candidate value, remains strictly inside $\mathcal Z$.
\end{remark}

To propagate the barrier constraints along the shifted candidate, the horizon CBF rows are enforced with stage-dependent margins

\begin{equation}
\label{eq:sigma}
\sigma_i {:=} \delta_{x,k} {+} \|\boldsymbol C\|_2\bigl(\varpi_{i+1}{+}(1{-}\alpha_{\mathrm{cbf}}) \varpi_i\bigr),
\; i{=}0,\ldots,N_p{-}1,
\end{equation}
i.e., \eqref{eq:mpc} enforces $h(\hat{\boldsymbol x}_{k+i+1|k})\ge(1-\alpha_{\mathrm{cbf}})h(\hat{\boldsymbol x}_{k+i|k})+\sigma_i$ with $\hat{\boldsymbol x}_{k|k}=\boldsymbol x_k$; since $\varpi_0=0$, the stage-$0$ row implies the single-step constraint~\eqref{eq:cbf_tight}, which it supersedes.

\begin{theorem}
\label{thm:joint}
Suppose Assumptions~\ref{as:operator_set}-\ref{as:initial} hold, and additionally $\mathcal Z_f\oplus\mathcal B(\bar\delta^+)\subseteq\mathcal Z^{\,t}_{N_p-1}$ and $\boldsymbol K_{\mathrm{ctrl}}\bigl(\mathcal Z_f\oplus\mathcal B(\bar\delta^+)\bigr)\subseteq\mathcal U^{\,t}_{N_p-1}$. Consider the MPC problem~\eqref{eq:mpc} posed with the tightened sets $\{\mathcal Z^{\,t}_i,\mathcal U^{\,t}_i\}$ of Lemma~\ref{lem:candidate}, the terminal constraint $\hat{\boldsymbol z}_{k+N_p|k}\in\mathcal Z_f$, and the CBF rows~\eqref{eq:sigma}, with measured re-initialization~\eqref{eq:mpc_init} and applied input $\boldsymbol u_k=\hat{\boldsymbol u}^{\ast}_{k|k}$. Then:
\begin{enumerate}
\item[\rm(i)] \textbf{Recursive feasibility:} the MPC problem remains feasible at every sampling instant.
\item[\rm(ii)] \textbf{Forward safety:} if $\delta_{x,k}=\delta^{\mathrm{ana}}_{x,k}$ in~\eqref{eq:sigma}, then $\boldsymbol x_k\in\mathcal S$ for all $k\ge0$.
\item[\rm(iii)] \textbf{Prediction-error boundedness and constraint margin:} $\|\boldsymbol z_{k+1}-\hat{\boldsymbol z}^{\ast}_{k+1|k}\|_2\le\delta^{\mathrm{ana}}_k$ for all $k$, hence
$\limsup_{k\to\infty}\|\boldsymbol z_{k+1}-\hat{\boldsymbol z}^{\ast}_{k+1|k}\|_2\le v_{\max}\nu_{\mathrm{eff}}/(1-\rho)$; moreover $\boldsymbol z_k\in\mathcal Z\ominus\mathcal B(\bar e)$ for all $k\ge1$.
\end{enumerate}
\end{theorem}

\begin{proof}\leavevmode

\textit{(i)} Suppose \eqref{eq:mpc} is feasible at time $k$ with optimizer $\boldsymbol U^{\ast}_k$ and nominal trajectory satisfying $\hat{\boldsymbol z}^{\ast}_{k+i|k}\in\mathcal Z^{\,t}_i$, $\hat{\boldsymbol u}^{\ast}_{k+i|k}\in\mathcal U^{\,t}_i$ for $i\le N_p-1$ and $\hat{\boldsymbol z}^{\ast}_{k+N_p|k}\in\mathcal Z_f$; this holds at $k=0$ by Assumption~\ref{as:initial}. Construct the candidate of Lemma~\ref{lem:candidate}. By Lemma~\ref{lem:candidate}(i) it satisfies the prediction dynamics and re-initialization, and by Lemma~\ref{lem:candidate}(iii) its stages $m=0,\ldots,N_p-2$ satisfy the tightened state and input constraints. For the terminal stage, the candidate state at stage $N_p-1$ is $\tilde{\boldsymbol z}:=\hat{\boldsymbol z}^{\ast}_{k+N_p|k}+\boldsymbol\zeta_{N_p-1}\in\mathcal Z_f\oplus\mathcal B(\bar\delta^+)$, since $b_m\le\bar\delta^+$ for all $m$ by~\eqref{eq:bm} and~\eqref{eq:horizon_uniform_margin}; the additional containments in the theorem statement place $\tilde{\boldsymbol z}\in\mathcal Z^{\,t}_{N_p-1}$ and the terminal candidate input $\kappa_f(\tilde{\boldsymbol z})=\boldsymbol K_{\mathrm{ctrl}}\tilde{\boldsymbol z}\in\mathcal U^{\,t}_{N_p-1}$, while Assumption~\ref{as:terminal} gives $(\hat{\boldsymbol A}_{k+1}+\hat{\boldsymbol B}_{k+1}\boldsymbol K_{\mathrm{ctrl}})\tilde{\boldsymbol z}\in\mathcal Z_f$.

For the CBF rows, write $\tilde{\boldsymbol x}_m:=\boldsymbol C(\hat{\boldsymbol z}^{\ast}_{k+1+m|k}+\boldsymbol\zeta_m)$. Assumption~\ref{as:barrier} gives $|h(\tilde{\boldsymbol x}_m)-h(\boldsymbol C\hat{\boldsymbol z}^{\ast}_{k+1+m|k})|\le L_h\|\boldsymbol C\|_2 b_m$, so the time-$k$ row at stage $m{+}1$ yields
$h(\tilde{\boldsymbol x}_{m+1})\ge(1-\alpha_{\mathrm{cbf}})h(\tilde{\boldsymbol x}_m)+L_{h} \sigma_{m+1}-L_h\|\boldsymbol C\|_2\bigl(b_{m+1}+(1-\alpha_{\mathrm{cbf}})b_m\bigr)
=(1-\alpha_{\mathrm{cbf}})h(\tilde{\boldsymbol x}_m)+L_h \sigma_m$,
using $\sigma_{m+1}-\sigma_m=\|\boldsymbol C\|_2\bigl(b_{m+1}+(1-\alpha_{\mathrm{cbf}})b_m\bigr)$ from~\eqref{eq:sigma} and $\varpi_{j+1}-\varpi_j=b_j$. Hence the candidate satisfies every constraint of the time-$(k{+}1)$ problem, completing the induction.

\textit{(ii)} The stage-$0$ row at time $k$ gives $h(\hat{\boldsymbol x}_{k+1|k})\ge(1-\alpha_{\mathrm{cbf}})h(\boldsymbol x_k)+L_h \sigma_0\ge(1-\alpha_{\mathrm{cbf}})h(\boldsymbol x_k)+L_h\delta^{\mathrm{ana}}_{x,k}$. Since $\boldsymbol x_{k+1}=\hat{\boldsymbol x}_{k+1|k}+\boldsymbol C\boldsymbol d_k$ with $\|\boldsymbol C\boldsymbol d_k\|_2\le\delta^{\mathrm{ana}}_{x,k}$, Assumption~\ref{as:barrier} yields
$h(\boldsymbol x_{k+1})\ge h(\hat{\boldsymbol x}_{k+1|k})-L_h\delta^{\mathrm{ana}}_{x,k}\ge(1-\alpha_{\mathrm{cbf}})h(\boldsymbol x_k)$,
where the first inequality is Assumption~\ref{as:barrier} and the second is the MPC constraint. Since $h(\boldsymbol x_0)\ge0$, induction gives $\boldsymbol x_k\in\mathcal S$ for all $k$.

\textit{(iii)} The first claim is~\eqref{eq:zeta0} with $\|\boldsymbol d_k\|_2\le\delta^{\mathrm{ana}}_k$ and the limit of $\bar E_k$ in~\eqref{eq:Ebar}. The second follows since $\boldsymbol z_{k+1}=\hat{\boldsymbol z}_{k+1|k+1}\in\mathcal Z^{\,t}_0=\mathcal Z\ominus\mathcal B(\bar e)$ by feasibility at time $k{+}1$ (part~(i)).
\end{proof}

\begin{remark}
\label{rem:computability}
Recursive feasibility in part~(i) depends only on the offline constants: $\bar e$ (Lemma~\ref{lem:rpi}), the offset sequence $\{\varpi_j\}_{j=0}^{N_p}$ (Lemma~\ref{lem:candidate}), and $\bar\delta^+$, all computed from $\bar\delta^{\mathrm{ana}}$, $\bar\mu$, $c$, and $\tilde\alpha$. The safety guarantee in part~(ii) additionally requires the online bound $\delta^{\mathrm{ana}}_k=v_{\max}\bar E_k$, evaluated with the conservative estimates of $\nu$ and $\|\boldsymbol E_0\|_F$ discussed in Section~\ref{sec:composite_bound}; it covers the analytical-tightening configuration, while the deployed conformal tightening is covered probabilistically by Theorem~\ref{thm:conformal}. Larger $c=\sqrt{\bar p/\underline p}$ inflates all margins and may introduce conservatism in geometrically constrained environments.
\end{remark}

\begin{remark}[Local Lipschitz constant]
\label{rem:local_Lh}
The global constant $L_h$ in Assumption~\ref{as:barrier} may be replaced by the local bound evaluated at the nominal predicted state $\hat{\boldsymbol x}_{k+1|k}=\boldsymbol C\hat{\boldsymbol z}_{k+1|k}$. For the circular barrier $h(\boldsymbol x)=\|\boldsymbol x-\boldsymbol c\|_2^2-d_{\mathrm{safe}}^2$,
\begin{equation}
\label{eq:local_Lh}
L_h^{\mathrm{local}}
=
2\bigl(\|\hat{\boldsymbol x}_{k+1|k}-\boldsymbol c\|_2+\delta_{x,k}^{\mathrm{ana}}\bigr),
\end{equation}
with $\delta_{x,k}^{\mathrm{ana}}=\|\boldsymbol C\|_2\,\delta_k^{\mathrm{ana}}$. Substituting $L_h^{\mathrm{local}}$ for $L_h$ preserves part~(ii), since the proof requires Assumption~\ref{as:barrier} only at the realized step.
\end{remark}

\begin{remark}
\label{rem:nominal_dyn}
The dynamics constraint~\eqref{eq:mpc_dyn} involves the nominal model $\hat{\boldsymbol{A}}_k$, $\hat{\boldsymbol{B}}_k$ only; the disturbance enters the formulation exclusively through $\delta_{x,k}$ in the CBF constraint~\eqref{eq:mpc_cbf}, which absorbs the combined effect of $\boldsymbol{d}_k$ and the model-update perturbation $\mu_k$ (Section~\ref{sec:composite_bound}). This separation between nominal prediction and disturbance handling is the defining feature of robust CBF-tightened MPC and underpins Theorem~\ref{thm:joint}.
\end{remark}

\section{Probabilistic Safety Guarantees}\label{sec:conf_pred}
\subsection{Conformal Safety Tightening}\label{sec:tightening}

As noted in Section~\ref{sec:anal_dist_bound}, the analytic margin
$\delta_k^{\mathrm{ana}}$ becomes overly conservative as $\rho\to1$, which is
precisely the correlated-regressor regime induced by tracking control. The
resulting tightened CBF constraints may then be infeasible in cluttered
environments, motivating a less conservative margin calibrated directly from
observed data.

The safe set is $\mathcal{S}:=\{\boldsymbol{x}\in\mathcal{X}: h(\boldsymbol{x})\geq0\}$
with $h:\mathbb{R}^{n_x}\to\mathbb{R}$ continuously differentiable. The true
state evolves as
$\boldsymbol{x}_{k+1}=\hat{\boldsymbol{x}}_{k+1|k}+\boldsymbol{C}\boldsymbol{d}_k$
with $\hat{\boldsymbol{x}}_{k+1|k}=\boldsymbol{C}(\hat{\boldsymbol{A}}_k\boldsymbol{z}_k
+\hat{\boldsymbol{B}}_k\boldsymbol{u}_k)$, so the nominal--true discrepancy is
governed by $\boldsymbol{s}_k:=\boldsymbol{C}\boldsymbol{d}_k$. We track its
slowly varying component by the exponential moving average (EMA)
\begin{equation}
\label{eq:ema}
\hat{\boldsymbol{s}}_k
=\alpha_{\mathrm{EMA}}\hat{\boldsymbol{s}}_{k-1}
+(1-\alpha_{\mathrm{EMA}})\boldsymbol{s}_k,
\qquad \alpha_{\mathrm{EMA}}\in(0,1),
\end{equation}
and define the nonconformity score as the one-step-ahead residual
\begin{equation}
\label{eq:score}
r_k:=\bigl\|\boldsymbol{s}_k-\hat{\boldsymbol{s}}_{k-1}\bigr\|_2 .
\end{equation}
Since $\hat{\boldsymbol{s}}_{k-1}$ is available before $\boldsymbol{s}_k$ is
observed, \eqref{eq:score} is a \emph{predictive} score: it quantifies the
component of the disturbance not anticipated by the EMA, and is therefore
admissible for conformal calibration.

Over a sliding calibration window $\mathcal{W}_k$ of size
$n_{\mathrm{conf}}\ge 1/\chi-1$ at risk level $\chi\in(0,1)$, the conformal
tightening scalar and its warm-up counterpart are
\begin{align}
\label{eq:delta_conf}
\delta_k^{\mathrm{conf}}
&:=\hat{Q}_{1-\chi}\bigl(\{r_j:j\in\mathcal{W}_k\}\bigr)
   +\bigl\|\hat{\boldsymbol{s}}_{k-1}\bigr\|_2,\\
\label{eq:delta_warmup}
\delta^{\mathrm{warm}}_k
&:=\hat{Q}_{1-\chi}\bigl(\{\|\boldsymbol{s}_j\|_2:j\in\mathcal{W}_k\}\bigr).
\end{align}
The additive term in \eqref{eq:delta_conf} is the current EMA magnitude, which
is measured online rather than assumed bounded; it accounts for any persistent
disturbance bias, such as that induced by a steady wind, without requiring
that bias to be small. During warm-up $k<k_{\mathrm{warm}}$, before the EMA has
converged, calibration is performed directly on the raw disturbance norms as
in \eqref{eq:delta_warmup}, for which $\|\boldsymbol{C}\boldsymbol{d}_k\|_2
=\|\boldsymbol{s}_k\|_2$ requires no bias correction. The tightening scalar
entering the CBF constraint \eqref{eq:mpc_cbf} is therefore
\begin{equation}
\label{eq:delta_implemented}
\sigma_{k}=
\begin{cases}
\delta^{\mathrm{warm}}_k & k<k_{\mathrm{warm}},\\[2pt]
\delta^{\mathrm{conf}}_k & k\ge k_{\mathrm{warm}},
\end{cases}
\end{equation}
calibrated entirely from observed residuals, with
$k_{\mathrm{warm}}\ge n_{\mathrm{conf}}$ so that the calibration window is
fully populated before \eqref{eq:delta_conf} is used. Note that $\sigma_k$ remains constant over the control horizon in MPC~\eqref{eq:mpc_cbf} for the the conformal implementation.

\subsection{Per-Step and Finite-Horizon Coverage}
\label{sec:probabilistic}

This subsection establishes safety guarantees using the probabilistic bounds
formulated in the previous subsection, providing a complementary counterpart
to the deterministic guarantees of Theorem~\ref{thm:joint}.

\begin{assumption}
\label{assm:exchangeable}
The scores $\{r_{k-n_{conf}}, \ldots, r_{k-1}, r_k\}$ are exchangeable, i.e.,
their joint distribution is invariant under permutation.
\end{assumption}

\begin{theorem}
\label{thm:conformal}
Let Assumptions~\ref{as:barrier} and~\ref{assm:exchangeable} hold for
$k \geq k_{warm}$, and let $n_{conf} \geq 1/\chi - 1$. Suppose that at each
step $k \geq k_{warm}$ the MPC enforces
\begin{equation}
    h\!\left(\boldsymbol{\hat{x}}_{k+1|k}\right)
    \geq (1-\alpha_{cbf})\,h(\boldsymbol{x}_k)
    + L_h \delta_k^{conf},
    \label{eq:cbf_prob}
\end{equation}
with $\delta_k^{conf}$ given by~\eqref{eq:delta_conf}. Then:
\begin{enumerate}
    \item \emph{Per-step guarantee:} For each $k \geq k_{warm}$,
    \begin{equation}
        P\!\left(
            h\!\left(\boldsymbol{x}_{k+1}\right)
            \geq (1-\alpha_{cbf})\,h(\boldsymbol{x}_k)
        \right) \geq 1 - \chi.
        \label{eq:per_step}
    \end{equation}
    \item \emph{Finite-horizon guarantee:} If $h(\boldsymbol{x}_{k_{warm}}) \geq 0$,
    then over any horizon of length $T < 1/\chi$,
    \begin{align}
        P\!\left(
            h(\boldsymbol{x}_k) \geq 0,\;
            \forall\, k \in \{k_{warm}{+}1,\ldots,k_{warm}{+}T\}
        \right) \nonumber \\ \geq 1 - T\chi.
        \label{eq:horizon}
    \end{align}
\end{enumerate}
\end{theorem}

\begin{proof}
\textbf{(i)} Under Assumption~\ref{assm:exchangeable}, the rank of $r_k$ among
$\{r_{k-n_{conf}}, \ldots, r_{k-1}, r_k\}$ is uniformly distributed over
$\{1,\ldots,n_{conf}{+}1\}$. By the standard conformal prediction coverage
theorem~\cite{vovk2005algorithmic} applied to the predictive residual scores
$\{r_j\}$,
\begin{equation}
    \label{eq:coverage}
    \mathbb{P}\Bigl(
        r_k \le \hat{Q}_{1-\chi}\bigl(\{r_j : j\in\mathcal{W}_k\}\bigr)
    \Bigr) \ge 1-\chi .
\end{equation}
Define the event
$\mathcal{E}_k := \bigl\{ r_k \le \hat{Q}_{1-\chi}(\{r_j : j\in\mathcal{W}_k\}) \bigr\}$.
On $\mathcal{E}_k$, the triangle inequality and the definition~\eqref{eq:score}
of the score give
\begin{align}
\label{eq:disturbance_bound}
\|\boldsymbol{C}\boldsymbol{d}_k\|_2
= \|\boldsymbol{s}_k\|_2
&\le \bigl\|\boldsymbol{s}_k - \hat{\boldsymbol{s}}_{k-1}\bigr\|_2
   + \bigl\|\hat{\boldsymbol{s}}_{k-1}\bigr\|_2 \nonumber\\
&= r_k + \bigl\|\hat{\boldsymbol{s}}_{k-1}\bigr\|_2
\;\le\; \delta_k^{\mathrm{conf}},
\end{align}
where the last inequality uses~\eqref{eq:delta_conf}. Note
that~\eqref{eq:disturbance_bound} holds without any assumption on the magnitude
of the EMA, since $\|\hat{\boldsymbol{s}}_{k-1}\|_2$ enters the margin as a
measured quantity.

Applying Assumption~\ref{as:barrier} with
$\boldsymbol{e} = \boldsymbol{C}\boldsymbol{d}_k$ and
substituting~\eqref{eq:disturbance_bound},
\begin{align}
    h\!\left(\boldsymbol{x}_{k+1}\right)
    &\geq h\!\left(\boldsymbol{\hat{x}}_{k+1|k}\right)
    - L_h\|\boldsymbol{C}\boldsymbol{d}_k\|_2 \nonumber \\
    &\geq h\!\left(\boldsymbol{\hat{x}}_{k+1|k}\right)
    - L_h\,\delta_k^{\mathrm{conf}}.
    \label{eq:lip_prob}
\end{align}
Substituting~\eqref{eq:cbf_prob} into~\eqref{eq:lip_prob} on $\mathcal{E}_k$,
\begin{align}
    h\!\left(\boldsymbol{x}_{k+1}\right)
    &\geq (1-\alpha_{cbf})\,h(\boldsymbol{x}_k)
    + L_h\,\delta_k^{\mathrm{conf}}
    - L_h\,\delta_k^{\mathrm{conf}} \nonumber \\
    &\geq (1-\alpha_{cbf})\,h(\boldsymbol{x}_k).
\end{align}
Since this holds on $\mathcal{E}_k$ and $P(\mathcal{E}_k) \geq 1-\chi$
by~\eqref{eq:coverage},~\eqref{eq:per_step} follows.

\textbf{(ii)} Define the failure event at each step as
$\mathcal{F}_k := \bigl\{ h(\boldsymbol{x}_{k+1})
< (1-\alpha_{cbf})\,h(\boldsymbol{x}_k) \bigr\}$.
From~(i), $P(\mathcal{F}_k) \leq \chi$ for each $k \geq k_{warm}$. By the union
bound,
\begin{equation}
    P\!\left(
        \bigcup_{k=k_{warm}}^{k_{warm}+T-1} \mathcal{F}_k
    \right)
    \leq \sum_{k=k_{warm}}^{k_{warm}+T-1} P(\mathcal{F}_k)
    \leq T\chi.
\end{equation}
On the complement event, the CBF evolution condition holds at every step, so
$h(\boldsymbol{x}_k) \geq (1-\alpha_{cbf})^{k-k_{warm}}
h(\boldsymbol{x}_{k_{warm}}) \geq 0$ for all
$k \in \{k_{warm}{+}1,\ldots,k_{warm}{+}T\}$, establishing~\eqref{eq:horizon}.

Theorem~\ref{thm:conformal} applies to $\delta_k^{warm}$ unchanged,
with~\eqref{eq:disturbance_bound} replaced by the direct bound
$\|\boldsymbol{C}\boldsymbol{d}_k\|_2 = \|\boldsymbol{s}_k\|_2
\leq \delta_k^{warm}$ on the corresponding coverage event, since the warm-up
scores are the raw disturbance norms and require no bias correction.
\end{proof}

\begin{remark}
\label{rem:feasibility}
The guarantee of Theorem~\ref{thm:conformal} requires
constraint~\eqref{eq:cbf_prob} to hold at each step. In practice, the CBF
evolution rows are implemented as soft constraints with penalty weight
$W_{cbf}$ to preserve solver feasibility in narrow passages. When the soft
constraint is violated, the per-step guarantee fails for that step. The
pointwise constraint $h(\boldsymbol{x}_k) \geq 0$ is additionally enforced as
a hard constraint for static obstacles; together with the tightening
$\sigma_{k}$ it provides the primary collision-avoidance mechanism, though,
being imposed on the predicted trajectory, the resulting certificate is
subject to the same per-step coverage as~\eqref{eq:per_step}.
\end{remark}

\begin{remark}[Local exchangeability]
\label{rem:local_exchangeability}
Assumption~\ref{assm:exchangeable} is standard in conformal
prediction~\cite{vovk2005algorithmic,angelopoulos2023conformal}. Although
exact exchangeability is generally violated in adaptive closed-loop systems,
its approximation is promoted by (i) calibrating over a short receding window
after the warm-up phase, (ii) scoring against the one-step-ahead EMA
prediction $\hat{\boldsymbol{s}}_{k-1}$, so that the score measures only the
unanticipated disturbance component and is insensitive to slowly varying
bias, and (iii) the contractive adaptation law
(Theorem~\ref{thm:contraction}), which renders the residual process
approximately stationary after convergence. Consequently, the recent
nonconformity scores are treated as locally exchangeable, allowing
$\delta_k^{\mathrm{conf}}$ to approximate the $(1-\chi)$ marginal score
quantile and thereby justify the approximate per-step coverage
in~\eqref{eq:per_step}. This assumption is local to the receding calibration
window and does not require exchangeability over the entire closed-loop
trajectory.
\end{remark}

\begin{algorithm}[t]
\caption{Safe Active Continual Koopman Control (SACK)}
\label{alg:SACK}
\begin{algorithmic}[1]

\Require Dataset $\mathcal{D}_{\mathrm{nom}}$, window $w$, forgetting factor $\gamma$, 
         step size $\eta$, horizon $N_p$, exploration weight $\beta$, 
         risk level $\chi$, warm-up length $k_{\mathrm{warm}}$
\Ensure Control input $\boldsymbol{u}_k^*$

\Statex
\State \textbf{Offline Training:}
\State Train autoencoder to obtain $\psi(\cdot)$, $\boldsymbol{A}$, $\boldsymbol{B}$, $\boldsymbol{C}$ using loss $\mathcal{L}_{\mathrm{nom}}$ \eqref{eq:total_loss}
\State Initialize $\hat{\boldsymbol{W}}_0 \gets [\boldsymbol{A}, \boldsymbol{B}]$

\Statex
\State \textbf{Online Execution:}
\For{$k = 1, 2, 3, \ldots$}

    \If{$k < w$}
        \State $\hat{\boldsymbol{W}}_k \gets \hat{\boldsymbol{W}}_0$
    \Else
        \State $\boldsymbol{z}_{k} \gets \psi(\boldsymbol{x}_{k})$
        \State Construct $\boldsymbol{V}_k$, $\boldsymbol{Z}^+_k$, $\boldsymbol{\Gamma}$ using \eqref{eq:Vk}--\eqref{eq:Zk}
        \State Compute Gramian $\boldsymbol{G}_k$ and prediction error $\boldsymbol{E}^{\mathrm{pred}}_k$
        
        \State Update $\hat{\boldsymbol{W}}_{k+1}$ via contractive adaptation law \eqref{eq:adaptation}
    \EndIf

    \If{using analytical bound}
    \State Compute disturbance bound $\sigma_k$ using \eqref{eq:sigma}
    \ElsIf {using conformal bounds}
    \If{$k < k_{\mathrm{warm}}$}
        \State $\sigma_k \gets \delta^{\mathrm{warm}}_k$ using \eqref{eq:delta_warmup}
    \Else
        \State Update EMA $\hat{\boldsymbol{s}}_k$ using \eqref{eq:ema}
        \State $\sigma_k \gets \delta^{\mathrm{conf}}_k$
    \EndIf
    \EndIf

    \State Solve D-optimal active-learning MPC problem \eqref{eq:mpc_obj} via SQP

    \State Apply $\boldsymbol{u}_k^*$ to system and observe $\boldsymbol{x}_{k+1}$
    \If {using conformal bounds}
    \State Update calibration window $\mathcal{W}_k$ with residual $r_k$.
    \EndIf
\EndFor
\end{algorithmic}
\end{algorithm}

\begin{figure*}[ht!]
     \centering
     \includegraphics[width=0.99\textwidth]{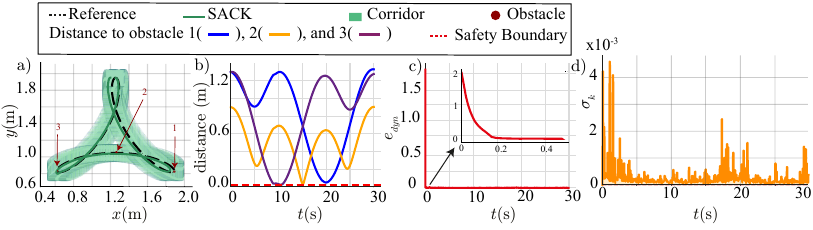}
     \caption{Constrained tracking performance for SACK for 3R manipulator within narrow corridors with obstacles. Red indicates breach of the corridor bounds. a) Traced path. b) Distance to obstacles. c) Evolution of the dynamic prediction error ($e_{dyn}$). d) Evolution of tightening scalar($\sigma_k$).}
    \label{fig:manip_perf_with_obs}
\end{figure*}

\begin{figure*}[h!]
         \centering
         \includegraphics[width=\textwidth]{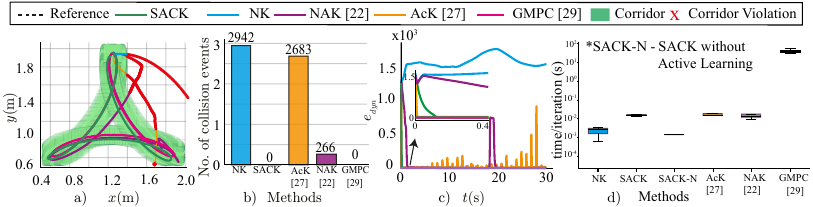}

     \caption{Comparison for constrained tracking performance for 3R manipulator within corridor for  SACK, NK, AcK~\cite{abraham2019active}, NAK~\cite{singh2025adaptive}, and GMPC~\cite{baltussen2025dual}. Red indicates breach of the corridor bounds. a) Tracking within the corridor. b) Collision count. c) Dynamic prediction error ($e_{dyn}$). d) Box plot for computation time per iteration.}
    \label{fig:manip_corr_comp}
\end{figure*} 


\begin{figure}[h!]
     \centering
     \includegraphics[width=0.5\textwidth]{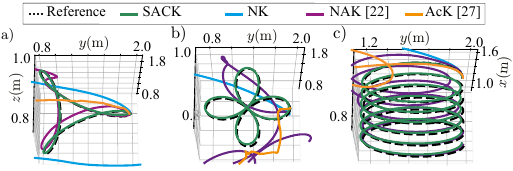}
\caption{Comparison of trajectory tracking performance of SACK, NK, AcK~\cite{abraham2019active}, and NAK~\cite{singh2025adaptive} for 3R manipulator across different shapes. a) Hypotrochoid.  b) Petal.  c) Helix.}
    \label{fig:manip_track_comp}
\end{figure} 

\begin{figure}[ht!]
     \centering
     \includegraphics[width=0.49\textwidth]{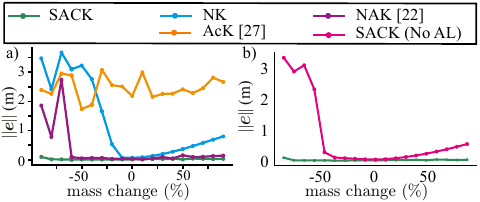}
     \caption{Tracking RMSE errors for tracking control of a 3R serial manipulator under mass variation of $-90\%$ to $90\%$. a) Comparison of SACK, NK, AcK~\cite{abraham2019active} and NAK~\cite{singh2025adaptive}. b) Comparison of the SACK with and without active learning.}
    \label{fig:manip_track_comp_mass_range}
\end{figure}

\section{Results and Discussion}
\label{sec:results}

We validate SACK through simulation on a 3R serial manipulator, a planar quadrotor, and a 7-DoF Franka Research 3 (FR3) in high-fidelity Gazebo, and through hardware experiments on TurtleBot3. In each case, the offline-trained model is evaluated under post-training distributional shifts induced by parameter variations and external disturbances. The performance of the proposed framework is evaluated using the tracking RMSE, the dynamic prediction error defined as,
\begin{equation}
e_{\mathrm{dyn},i}=\sqrt{\tfrac{1}{N_d}\textstyle\sum_{j=1}^{N_d}\|\hat{\bm{x}}_{i+j|i}-\bm{x}_{i+j}\|^{2}},
\label{eq:edyn}
\end{equation}
settling time $t_s$ (defined as the time required for $e_{\mathrm{dyn}}$ to enter and remain within $10\%$ of steady-state value), and constraint satisfaction. All controllers use full-state feedback and are implemented using  the ACADOS solver \cite{acados2021} (Intel\textsuperscript{\textregistered} 
Core\textsuperscript{\texttrademark} i7-10700, 16GB RAM, Nvidia\textsuperscript{\textregistered} 
Geforce\textsuperscript{\texttrademark} RTX 4070 Ti). We compare against four baselines: nominal Koopman with no adaptation (NK), neural adaptive Koopman (NAK)~\cite{singh2025adaptive}, EDMD-based active learning (AcK)~\cite{abraham2019active}, and GP-based dual MPC (GMPC)~\cite{baltussen2025dual}. Details of Offline training, network architectures, and hyperparameters are provided in Appendix~\ref{app:offline_train}.
 
\subsection{3R Serial Manipulator}
\label{sec:serial_manipulator}

The manipulator follows the standard rigid-body dynamics $\bm{M}(\bm{\theta})\ddot{\bm{\theta}}+\bm{C}(\bm{\theta},\dot{\bm{\theta}})\dot{\bm{\theta}}+\bm{G}(\bm{\theta})=\bm{\tau}$, where
$\boldsymbol{M}$, $\boldsymbol{C}$, and $\boldsymbol{G}$ 
represent the mass, Coriolis, and gravity matrices, respectively. We train the nominal model for link mass $m_i = 0.6~\text{kg}$, length $l_i = 1~\text{m}$, and inertia $I_i = \text{diag}[0, \frac{m_i l_i^2}{12}, \frac{m_i l_i^2}{12}]~\text{kgm}^2$ for $i = 1,2,3$. We introduce distribution shifts by modifying the link masses.

\subsubsection{Task 1: Safe exploration in a constrained corridor}
 We first evaluate SACK in a narrow obstacle-populated corridor where the mass of each link is decreased by  $40\%$. We fix the exploration weight at $\beta = 10^3$, and use conformal tightening to handle transient uncertainty during warmup. Figure~\ref{fig:manip_perf_with_obs} shows that SACK traverses the corridor safely without collisions. The executed trajectory deviates from the reference to explore informative regions, consistent with the active learning objective. By improving the conditioning of the adaptation problem, these information-rich trajectories enable the contractive update to rapidly reduce the model error, achieving a settling time of $t_s \approx 0.1~\mathrm{s}$ (Fig.~\ref{fig:manip_perf_with_obs}c). The tightening scalar $\sigma$ also decreases rapidly as adaptation progresses (Fig.~\ref{fig:manip_perf_with_obs}d), reflecting the reduction in model mismatch and the corresponding safety margin.

Next, we compare SACK against all the baselines under a more severe $60\%$ mass reduction with $\beta = 10^4$ in an obstacle-free corridor. By design, note that all reference trajectories remain collision-free, so any controller that adapts sufficiently fast should remain inside the corridor without requiring an explicit safety mechanism. However, only SACK and GMPC~\cite{baltussen2025dual} remain inside the corridor. NK, NAK~\cite{singh2025adaptive}, and AcK~\cite{abraham2019active} all violate corridor constraints (Fig.~\ref{fig:manip_corr_comp}). SACK adapts with $t_s\approx0.12$\,s, roughly ten times faster than NAK ($t_s\approx1.17$\,s), whose gradient-based updates additionally exhibit a secondary error spike at $t=18.49$\,s, which indicates that closed-loop data fails to excite all regressor directions, so corrections don't generalize across the workspace under large distributional shift. AcK~\cite{abraham2019active} exhibits unsafe behavior because its windowed EDMD re-estimation is sensitive to local data conditioning and carries no convergence guarantee. GMPC~\cite{baltussen2025dual} remains safe, but its per-iteration cost is roughly four orders of magnitude higher than SACK's (Fig.~\ref{fig:manip_corr_comp}d) on account of the computational burden incurred by nonlinear MPC and online GP updates,  precluding real-time deployment. Removing the information objective (SACK-N) reduces runtime for SACK to the order of NK, showing that the law itself introduces negligible computational overhead and the bulk of the computational burden is attributable to the active-learning term. Due to the low cost of the adaption law, SACK, even with the burden of the information active learning objective, achieves computational cost comparable to NAK~\cite{singh2025adaptive} and AcK~\cite{abraham2019active},  enabling real-time deployment while providing substantially improved robustness and safety performance under large distribution shift.

\subsubsection{Task 2: Tracking under parametric mismatch}
To evaluate adaptation independently of obstacle avoidance, all obstacles and corridor constraints are removed. Then, we sweep the system mass from $- 90\%$ to ${+}90\%$ while tracking hypotrochoid, petal, and helix reference trajectories. We initialize $\beta=10^{3}$ and gradually decay its value as the model converges, thereby shifting the controller from exploration toward tracking precision. GMPC~\cite{baltussen2025dual} is excluded because its high computational cost renders real-time deployment infeasible. For a representative case of $-60\%$ shift (Fig.~\ref{fig:manip_track_comp}), SACK attains RMS tracking errors of $0.0151$, $0.0159$, and $0.0324$\,m on the three shapes, respectively, whereas all the baselines show significantly degraded performance. For the full range of distribution shift (Fig.~\ref{fig:manip_track_comp_mass_range}a), SACK consistently maintains an RMSE of the order of $10^{-2}$ m, whereas NK exceeds $3$\,m under moderate-to-large shifts. NAK~\cite{singh2025adaptive} exhibits inconsistent behaviour with large error spikes that reflect its sensitivity to gradient-update quality. AcK~\cite{abraham2019active} shows consistently poor performance. Likewise, to isolate the contribution of active learning, we ablate the information objective by setting $\beta = 0$ while retaining the adaptation law and safety constraints (Fig.~\ref{fig:manip_track_comp_mass_range}b). The ablated variant degrades substantially under moderate-to-large shifts, indicating that the closed-loop trajectory alone does not sufficiently condition the regression problem demonstrating that active learning is not merely 
complementary but necessary for reliable adaptation in 
practice.

\begin{figure*}[h!]
         \includegraphics[width=\textwidth]{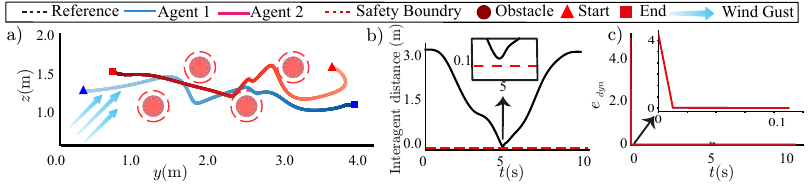}
\caption{Cooperative navigation of two planar quadrotors under a $20\%$
    mass increase and a $3$\,m/s wind disturbance. (a) Executed trajectories
    through the shared obstacle field. (b) Inter-agent distance over time.
    (c) Dynamic prediction error $e_{\mathrm{dyn}}$.}
    \label{fig:two_quad_obs}
\end{figure*} 

\begin{figure*}[h!]
         \includegraphics[width=\textwidth]{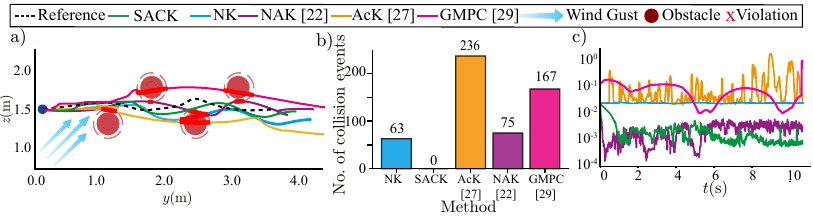}
     \caption{Comparison for constrained trajectory tracking performance for planar quadrotor within an obstacle field for SACK, NK, AcK~\cite{abraham2019active}, NAK~\cite{singh2025adaptive},  and GMPC~\cite{baltussen2025dual}. Red indicates violation of the safety boundary. a) Traced Paths b) Number of collision events c) Dynamic prediction error ($e_{dyn}$).}
    \label{fig:planar_quad_obs_comp}
\end{figure*} 

\begin{figure*}[h!]
         \centering
         \includegraphics[width=\textwidth]{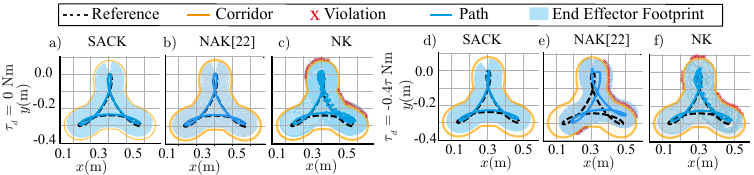}

     \caption{Tracking performance for the Franka Research 3 for the SACK (a,d), NAK~\cite{singh2025adaptive} (b,e), and  NK (c,f) within a corridor in the presence (d,e,f) and absence (a,b,c) of input disturbance. The end effector blueprint is highlighted in blue.}
    \label{fig:franka_setup}
\end{figure*} 

\subsection{Planar Quadrotor}
We next consider a planar quadrotor with aerodynamic wind disturbances. The dynamics is given by \cite{singh2025adaptive}:

\begin{align}
    \label{eq:planar_quad_dynamics}
    \begin{bmatrix}
        \ddot{y} \\ \ddot{z} \\ \ddot{\theta} 
    \end{bmatrix}
    &=
    \begin{bmatrix} 0 \\ -g \\ 0 \end{bmatrix}
    +
    \begin{bmatrix}
        -\frac{1}{m}\sin\theta & -\frac{1}{m}\sin\theta \\
         \frac{1}{m}\cos\theta &  \frac{1}{m}\cos\theta \\
        -\frac{l_{\text{arm}}}{I} & \frac{l_{\text{arm}}}{I}
    \end{bmatrix}
    \begin{bmatrix} T_1 \\ T_2 \end{bmatrix}
    + \frac{\boldsymbol{F}_w}{m}, \nonumber \\
    \boldsymbol{F}_w
    &=
    \begin{bmatrix}
        K v_w^2 \cos(\alpha_w) &
        K v_w^2 \sin(\alpha_w) &
        0
    \end{bmatrix}^{\top},
\end{align}
where $T_1$ and $T_2$ denote the thrust inputs; 
$m$, $l_{\text{arm}}$, $I$, and $g$ denote the mass, 
arm length, rotational inertia, and gravitational 
acceleration, respectively; and $\boldsymbol{F}_w$ represents the wind disturbance with speed $v_w$, direction $\alpha_w$, and drag coefficient $K$. The Koopman model is trained using data collected under nominal operating conditions with no wind ($v_w = 0$). The corresponding system parameters are
$m = 2~\mathrm{kg}$,
$I = 1~\mathrm{kg\,m^2}$,
$g = 9.81~\mathrm{m/s^2}$,
$l_{\mathrm{arm}} = 0.2~\mathrm{m}$,
and $K = 0.1~\mathrm{kg/m}$. Unlike the manipulator experiments, the quadrotor experiences time-varying wind disturbances that induce a drifting perturbation in the true Koopman operator with $\nu > 0$.

We evaluate SACK on a cooperative navigation task involving two planar quadrotors traversing a shared static obstacle field under a $+20\%$ mass increase and a $3~\mathrm{m/s}$ wind disturbance. The exploration weight is set to $\beta = 10$, since the coupled quadrotor dynamics already provide sufficient excitation to keep the regressor Gramian well conditioned without aggressive exploration. For implementation, each agent solves its own MPC and adapts its own Koopman operator independently. Both agents complete the task without obstacle or inter-agent collisions (Fig.~\ref{fig:two_quad_obs}a), while maintaining inter-agent clearance above the safety threshold throughout (Fig.~\ref{fig:two_quad_obs}b). Consistent with this excitation-rich setting, the dynamic prediction error $e_{\mathrm{dyn}}$ settles rapidly, with $t_s \approx 0.05~\mathrm{s}$, indicating fast adaptation under the imposed shift (Fig.~\ref{fig:two_quad_obs}c). 

Next, for a single-quadrotor case under a severe distribution shift ($+30\%$ mass, $5~\mathrm{m/s}$ wind), only SACK maintains safe obstacle-constrained tracking (Fig.~\ref{fig:planar_quad_obs_comp}). GMPC~\cite{baltussen2025dual} degrades in this setting because the uncertainty is no longer dominated by a structured parametric shift, as in the manipulator case, but by the combined effect of mass mismatch and exogenous wind producing residuals that are difficult to model as a stationary state-input-dependent GP, particularly when wind is not included in the regression input. As a result, GP adaptation fails to compensate for the mismatch sufficiently quickly, leading to obstacle-constraint violations. In contrast, SACK effectively adapts the Koopman operator directly from lifted prediction errors through a closed-form contractive update, enabling it to compensate for the combined perturbations without explicitly separating their sources.

\subsection{Franka Research 3 (Gazebo)}
Next we consider a 7-DoF Franka Research~3 (FR3) serial arm, where the end effector must track a reference in a narrow corridor while keeping its entire footprint inside the admissible region, a set-valued geometric constraint stricter than the point constraints considered prior. GMPC~\cite{baltussen2025dual} and AcK~\cite{abraham2019active} are excluded following computational infeasibility and poor performance, respectively, as demonstrated in previous sections. Under nominal conditions, SACK (Fig.~\ref{fig:franka_setup}a) tracks accurately with no unnecessary exploratory deviation. Further, under a $40\%$ resistive joint-torque disturbance that introduces a structured distributional shift absent from the training data, it again maintains safe corridor-constrained tracking as the contractive adaptation law compensates for model mismatch and the conformal tightening scalar adapts to the residual disturbance online (Fig.~\ref{fig:franka_setup}d). In contrast, NAK~\cite{singh2025adaptive} exits the corridor in the perturbed case (Fig.~\ref{fig:franka_setup}e) and NK fails in both settings (Fig.~\ref{fig:franka_setup}c and Fig.~\ref{fig:franka_setup}f). These set of simulations demonstrate the scalability of SACK to higher-dimensional systems.

\begin{figure}[ht!]
     \centering
     \includegraphics[width=0.45\textwidth]{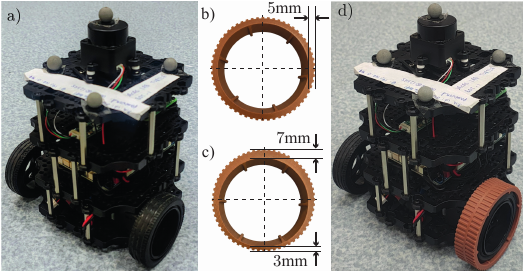}
     \caption{Experimental TurtleBot3 platform. (a) Nominal TurtleBot3 configuration. (b) Axisymmetric wheel attachment. (c) Eccentric wheel attachment. (d) TurtleBot3 with the wheel attachment mounted on one of the wheels.}
    \label{fig:tb_setup}
\end{figure}

\begin{table}[ht!]
\small\sf\centering
\caption{\label{tab:expt_parameters} Experimental statistics: average RMSE (m) and average number of time steps in unsafe set; over 20 trials for open-circuit (OC) and closed-circuit (CC) tracks.}
\begin{tabular}{|p{1.8cm}|p{1.8cm}|p{1.8cm}|p{1.8cm}|} 
\hline
 & SACK & NAK~\cite{singh2025adaptive} & NK \\ 
\hline
OC & \textbf{0.244, 0} & 1.59, 384 & 1.26, 76 \\ 
\hline
CC & \textbf{0.112, 0} & 0.984, 389 & 1.329, 669 \\ 
\hline
\end{tabular}
\end{table}
\begin{figure*}[h!]
     \includegraphics[width=\textwidth]{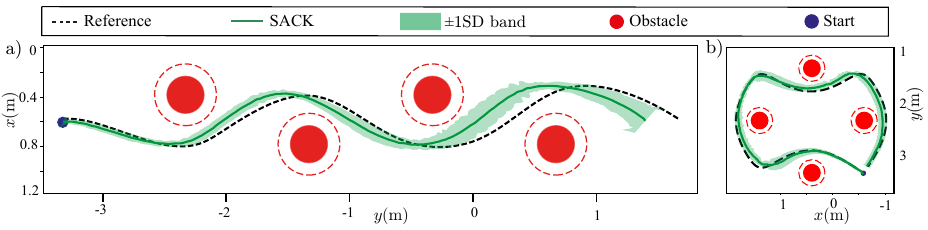}
    \caption{Trajectory distribution over 20 hardware trials across all three deployment configurations (nominal, axisymmetric ring, eccentric ring) for SACK. The solid line shows the mean trajectory, and the shaded band denotes the $\pm 1SD$ corridor computed in the path-normal direction. a) open circuit b) closed circuit.}
    \label{fig:tb_monte_caro}
\end{figure*}

\begin{figure}[h!]
         \centering
         \includegraphics[width=0.49\textwidth]{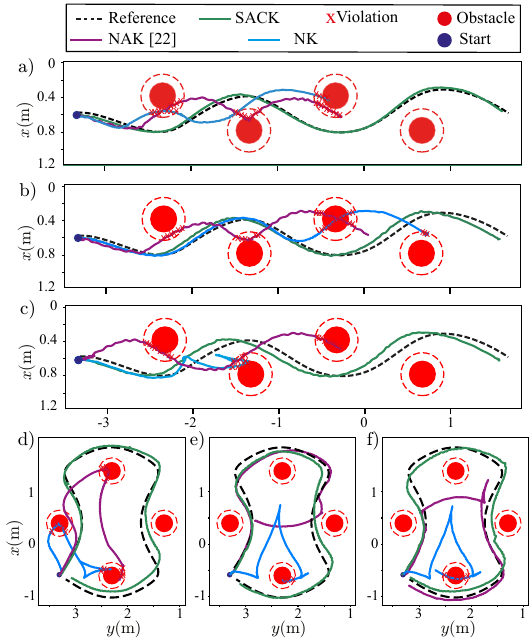}

     \caption{Experimental comparison for constrained trajectory tracking performance for TurtleBot3 robot within an obstacle field for SACK, NK, and NAK~\cite{singh2025adaptive} within open (a-c) and closed circuit (d-f) paths. (a,d) Nominal robot configuration. (b,e) Robot with a axisymmetric circular ring mounted on left wheel. (c,f) robot with an eccentric ring mounted on left wheel.}
    \label{fig:tb_expts}
\end{figure} 

\subsection{Hardware Experiments: TurtleBot3}
We finally deploy SACK on a TurtleBot3 Burger mobile robot~(Fig.~\ref{fig:tb_setup}) with the nominal model trained offline in Gazebo, so that online adaptation must account for the sim-to-real gap and induced unmodeled hardware effects. Three configurations of increasing mismatch are considered: a nominal platform, an axisymmetric wheel ring attached to a wheel, effectively changing the wheel radius and introducing an asymmetric kinematic mismatch, and an eccentric wheel ring, which induces nonstationary disturbances through periodic uneven rolling. Together, these configurations emulate payload asymmetry, uneven wheel wear, actuator imbalance, and irregular terrain effects, and are evaluated on an open-circuit~(OC) and a closed-circuit~(CC) obstacle course.

Table~\ref{tab:expt_parameters} and Fig.~\ref{fig:tb_monte_caro} together characterize the performance of SACK's in all configurations and circuit types of deployment in $20$ trials. SACK achieves zero unsafe-set violations with lower RMSE than both baselines, showing that safety is maintained consistently throughout individual trials and not just on average. The OC-CC comparison further illustrates how performance improves as online adaptation accumulates data: the CC track is longer, allowing SACK to collect more informative closed-loop samples and progressively refine the Koopman model during deployment, reflected in the reduced RMSE and the tightening trajectory spread visible in the last portion of Fig.~\ref{fig:tb_monte_caro}.

To contextualize these aggregate results, Fig.~\ref{fig:tb_expts} shows representative single-trial trajectories for SACK and both baselines across the three deployment configurations. Even under the nominal configuration, NK and NAK~\cite{singh2025adaptive} violate the safety margins due to the sim-to-real gap alone, and degrade further as the mismatch grows. In contrast, SACK completes all runs collision-free across every configuration, confirming that contractive adaptation, active excitation, and conformal safety tightening together enable reliable deployment under significant real-world model mismatch.

\section{Conclusion}\label{sec:conclusion}
This work demonstrates that continual Koopman learning can be integrated with active exploration and formal safety guarantees within a unified model predictive control framework. Rather than treating online adaptation, informative data collection, and safety-critical control as separate objectives, SACK jointly optimizes them to enable continual model refinement during deployment while preserving recursive feasibility and safe closed-loop operation under distributional shift. The proposed theoretical framework establishes convergence of the online adaptation law under persistent excitation together with deterministic and probabilistic safety guarantees, providing a principled foundation for safe continual model-based learning. Extensive simulation and experimental analysis demonstrate that these theoretical properties translate into improved prediction accuracy, tracking performance, and robustness across diverse robotic platforms and sources of model mismatch. Collectively, these results suggest that SACK provides a practical step toward deploying adaptive Koopman-based controllers in long-term robotic autonomy, where models must continually evolve while maintaining the safety and reliability required for real-world operation.

Several directions remain open for future work. Extending the adaptation framework to jointly update the lifting map $\psi(\cdot)$ alongside the Koopman operator matrices would relax the fixed-subspace assumption and improve robustness under severe distributional shifts. Further analysis of conformal coverage under correlated closed-loop residuals would strengthen the probabilistic safety guarantees. Improving the computational scalability of the SQP-based solver for higher-dimensional robotic platforms and extending the framework to cooperative multi-agent settings also represent promising directions for future research.

\bibliographystyle{ieeetr}
\bibliography{citation.bib}

\appendices
\section{Offline Training and Implementation Details}
\label{app:offline_train}

\subsection{Neural Network-Based Offline Learning}
\label{app:offline_learning}
The offline training phase produces the three quantities that the online adaptive module requires: the lifting map $\psi(\cdot)$, the nominal Koopman matrices $(\boldsymbol{A}, \boldsymbol{B})$, and the reconstruction matrix $\boldsymbol{C}$. These are learned jointly from a nominal input-output dataset $  \mathcal{D}_{\mathrm{nom}}
    = \bigl\{(\boldsymbol{X}_i,\, \boldsymbol{Y}_i,\, \boldsymbol{U}_i)
      \bigr\}_{i=1}^{M},$ where $\boldsymbol{X}_i$ and $\boldsymbol{Y}_i$ are matrices of consecutive
state observations and $\boldsymbol{U}_i$ the corresponding inputs, collected under nominal operating conditions. Once training is complete, $\psi(\cdot)$, and $\boldsymbol{C}$ are held fixed for the remainder of deployment. The quality of the offline-learned $\psi(\cdot)$ determines the validity of Assumption~\ref{as:fixed_lifting}: if the nominal training data provides adequate coverage of the observable subspace, the perturbed dynamics will remain representable within the span of $\{\phi_1, \ldots, \phi_p\}$, and online adaptation over $(\boldsymbol{A}, \boldsymbol{B})$ alone will be sufficient to compensate for distributional shift.


We adopt the autoencoder-based Koopman architecture
of~\cite{singh2025adaptive}, illustrated in Fig.~\ref{fig:active_block}. In this architecture, the encoder network realizes the lifting map $\boldsymbol{z}_k = \psi(\boldsymbol{x}_k) \in \mathbb{R}^p$, a linear layer learns the Koopman operator matrices $\boldsymbol{A}$ and $\boldsymbol{B}$ governing the lifted dynamics
$\boldsymbol{z}_{k+1} = \boldsymbol{A}\boldsymbol{z}_k
  + \boldsymbol{B}\boldsymbol{u}_k$, and a linear decoder realizes the reconstruction $\boldsymbol{\hat{x}}_k = \boldsymbol{C}\boldsymbol{z}_k$. The autoencoder structure is natural for this problem because it enforces the Koopman requirement that the lifted state be both dynamically consistent, evolving linearly under $(\boldsymbol{A}, \boldsymbol{B})$, and physically interpretable, decodable back to the original state space via $\boldsymbol{C}$. The network is trained end-to-end by minimizing a composite loss that encodes the three structural requirements of a valid Koopman representation:
\begin{equation}
  \mathcal{L}_{\mathrm{nom}}
    {=} \alpha_1 \mathcal{L}_{\mathrm{rec}}
    {+} \alpha_2 \mathcal{L}_{\mathrm{pred}}
    {+} \alpha_3 \mathcal{L}_{\mathrm{lift}}
    {+} \gamma_1 \|\boldsymbol{\Theta}\|_1
    {+} \gamma_2 \|\boldsymbol{\Theta}\|_2,
  \label{eq:total_loss}
\end{equation}
where $\boldsymbol{\Theta}$ denotes all trainable network parameters and
$\alpha_1, \alpha_2, \alpha_3, \gamma_1, \gamma_2 > 0$ are weighting
hyperparameters. The three loss terms are: $\mathcal{L}_{\mathrm{rec}}
    = \bigl\|\boldsymbol{x}_k - \boldsymbol{C}\boldsymbol{z}_k\bigr\|^2$, $\mathcal{L}_{\mathrm{pred}}
    = \bigl\|\boldsymbol{x}_{k+1} - \boldsymbol{C}\boldsymbol{\hat{z}}_{k+1|k}\bigr\|^2$ and $\mathcal{L}_{\mathrm{lift}}
    = \bigl\|\boldsymbol{z}_{k+1} - \boldsymbol{\hat{z}}_{k+1|k}\bigr\|^2$,
where $\boldsymbol{\hat{z}}_{k+1|k}
  := \boldsymbol{A}\boldsymbol{z}_k + \boldsymbol{B}\boldsymbol{u}_k$
is the one-step linear prediction in the lifted space.
$\mathcal{L}_{\mathrm{rec}}$ enforces that $\boldsymbol{C}$ is a valid decoder, i.e., that $\boldsymbol{C}\psi(\boldsymbol{x}) \approx \boldsymbol{x}$ for all $\boldsymbol{x}$ in the training distribution. $\mathcal{L}_{\mathrm{pred}}$ enforces that the linear lifted dynamics predict the next physical state accurately. $\mathcal{L}_{\mathrm{lift}}$ enforces the Koopman invariance condition directly in the lifted space; it is the loss term most tightly coupled to the requirement that $\mathrm{span}\{\phi_1, \ldots, \phi_p\}$ be approximately invariant under the system flow. The $L_1$ and $L_2$ regularization terms penalize all trainable parameters $\boldsymbol{\Theta}$ to reduce overfitting. For ease of learning, we include the base state, $\boldsymbol{x}$ as a part of the lifted state, so $\boldsymbol{z} = [\boldsymbol{x}^{\top},\; \psi_{n_x+1},\; \ldots,\; \psi_p]^\top$. In this case the $\boldsymbol{C}$ matrix simplifies to $\boldsymbol{C} = [\boldsymbol{I}_{n_x \times n_x}^\top  \boldsymbol{0}_{n_x \times (p-n_x)}^\top]^\top$ and the loss $\mathcal{L}_{\mathrm{rec}}$ becomes redundant. The specific values of all hyperparameters for different robotic platforms are tabulated in Table \ref{tab:nn_params}.


Upon convergence, offline training yields the tuple
$\{\psi(\cdot),\, \boldsymbol{A},\, \boldsymbol{B},\, \boldsymbol{C}\}$. The matrices $(\boldsymbol{A}, \boldsymbol{B})$ initialize the adaptive estimates: $\boldsymbol{\hat{A}}_0 = \boldsymbol{A}$ and
$\boldsymbol{\hat{B}}_0 = \boldsymbol{B}$, so that the initial parameter estimation error is
$\boldsymbol{E}_0 = \boldsymbol{W}^*_0 - \boldsymbol{\hat{W}}_0
  = [\Delta\boldsymbol{A}_0,\, \Delta\boldsymbol{B}_0]$.
The magnitude $\|\boldsymbol{E}_0\|_F$ quantifies the initial mismatch between the nominal and true Koopman operators at deployment time and appears explicitly in the finite-time convergence bound of Theorem~\ref{thm:contraction}. A higher-quality offline model, one trained on data that better covers the deployment distribution, yields a smaller $\|\boldsymbol{E}_0\|_F$ and therefore faster convergence of the online adaptation law.

\subsection{Implementation Details}
In all cases, the Koopman matrices $\boldsymbol{A}$, $\boldsymbol{B}$, $\boldsymbol{C}$ are identified from the collected data using a neural network architecture. The hyperparameters for the employed architecture are shown in Table~\ref{tab:nn_params}. In Table \ref{tab:nn_params}, the architecture of the 3R manipulator is written as $[6,30,30,17]$, which basically means that the encoder has an input layer of dimension $6$ corresponding to each state, there are two hidden layers, each of width $30$, and the output dimension is $17$, which basically is the dimension of the lifted states. The same nomenclature holds for the other systems. 

\begin{table}[ht!]
\small\sf\centering
\caption{\label{tab:nn_params} Hyperparameters of the Koopman network and CBF implemetatoion.}
    \begin{tabular}{|p{1 cm}|p{1.3 cm}|p{1.3cm}|p{1.3 cm}|p{1.3cm}| } 
     \hline
      & {3R manipulator} & {Planar Quadrotor}& {Franka Research 3} & {TurtleBot 3}\\ 
     \hline
     \multicolumn{5}{|c|} {Nominal Koopman Autoencoder (linear)}\\
     \hline
    {Archite- cture} & [6, 30, 30, 17] & [6, 20, 20, 17] & [14, 30, 30, 73] & [3, 30, 30, 14]\\ \hline
    { \# lifted state} & 17 & 17 & 73 & 14\\\hline
     {$\alpha_1$}, {$\alpha_2$}, {$\alpha_3$} & 1, 0.3, 1 & 1, 0.3, 1 & 1, 0.5, 1 & 1, 0.3, 1\\  \hline
     {$\gamma_1$}, {$\gamma_2$} & $0$, $0$ & 0, 0 & $0$, $0$ & $0$, $0$ \\ \hline
     {Batch Size} & 256 & 256 & 256 & 256\\
     \hline
      \multicolumn{5}{|c|} {CBF Implementation}\\
      \hline
    {$n_{conf}$} & 40 & 40 & 40 & 40\\ \hline 
    {$\chi$} & 0.025 & 0.025 & 0.025 & 0.025\\ \hline
    {$\alpha_{cbf}$} & 0.1 & 0.1 & 0.1 & 0.1 \\ \hline
    {$\alpha_{EMA}$} & 0.2 & 0.2 & 0.2 & 0.2  \\ \hline
       
     \hline
    \end{tabular}

\end{table}

\subsection{Data Generation}
Nominal Koopman models are identified offline from simulation data prior to deployment. For the 3R serial manipulator, planar quadrotor, and Franka Research 3, training trajectories are generated using minimum-snap trajectory optimization through randomly sampled waypoints. Segment timing is optimized under a maximum velocity constraint, and $10^{\text{th}}$-order polynomial coefficients are solved via 
unconstrained endpoint optimization, producing $C^4$-continuous reference trajectories that excite a broad range of configurations and velocities. The system is simulated under these references and state-input pairs are logged at the control frequency.

For the TurtleBot3 Burger, velocity command trajectories are collected in Gazebo under a combination of randomized inputs and sinusoidal velocity profiles spanning the full admissible range $\boldsymbol{v} \in [-0.176, 0.176]$\,m/s, $\omega \in [-2.272, 2.272]$\,rad/s. The sinusoidal profiles ensure smooth, persistently exciting trajectories that cover the nonlinear kinematic regime, while the random inputs provide 
broad coverage of the input space.

\section{Additional Theoretical Analysis}
\label{app:active_learn_better}

\begin{theorem}
\label{thm:perturbation}
Consider the nominal MPC
\begin{equation*}
  (P_0)\colon\quad
  \boldsymbol{U}^{\mathrm{nom}}
    \in \arg\min_{\boldsymbol{U}\in\mathcal{F}}\,
        J_{\mathrm{task}}(\boldsymbol{U}),
\end{equation*}
and the active-learning MPC
\begin{equation*}
  (P_\beta)\colon\quad
  \boldsymbol{U}^{\mathrm{AL}}(\beta)
    \in \arg\min_{\boldsymbol{U}\in\mathcal{F}}\,
        \bigl(J_{\mathrm{task}}(\boldsymbol{U})
              - \beta\,J_{\mathrm{info}}(\boldsymbol{U})\bigr),
  \; \beta \geq 0,
\end{equation*}
where $\mathcal{F}\subset\mathbb{R}^{n_u N_p}$ is a nonempty feasible set
encoding dynamics, input and state constraints, and (if applicable) CBF
inequalities, and
$J_{\mathrm{info}}(\boldsymbol{U})
  := \log\det\!\bigl(
       {\boldsymbol{G}}_k^{\mathrm{pred}}(\boldsymbol{U}) + \varepsilon\boldsymbol{I}
     \bigr)$
with ${\boldsymbol{G}}_k^{\mathrm{pred}}(\boldsymbol{U})
  = {\boldsymbol{V}}_k^{\mathrm{pred}}(\boldsymbol{U})
    {\boldsymbol{V}}_k^{\mathrm{pred}}(\boldsymbol{U})^\top$
and $\varepsilon > 0$.
Suppose that the following conditions hold at $\boldsymbol{U}^{\mathrm{nom}}$:
\begin{enumerate}
  \item \emph{(Interior feasibility)}
        $\boldsymbol{U}^{\mathrm{nom}}$ is a strict local minimizer of
        $(P_0)$ and lies in the interior of $\mathcal{F}$, i.e., there
        exists a neighbourhood $\mathcal{N}$ of $\boldsymbol{U}^{\mathrm{nom}}$
        such that $\mathcal{N}\subset\mathcal{F}$.
  \item \emph{(Smoothness)}
        Both $J_{\mathrm{task}}$ and $J_{\mathrm{info}}$ are twice
        continuously differentiable on $\mathcal{N}$.
  \item \emph{(Positive-definite Hessian)}
        $\boldsymbol{H}_{\mathrm{task}}
          {:=} \nabla^2 J_{\mathrm{task}}(\boldsymbol{U}^{\mathrm{nom}})
          \succ \boldsymbol{0}$.
\end{enumerate}
Then there exist $\bar{\beta} > 0$ and a differentiable map
$\beta\mapsto\boldsymbol{U}^{\mathrm{AL}}(\beta)$ for
$\beta\in[0,\bar{\beta}]$, with $\boldsymbol{U}^{\mathrm{AL}}(0) =
\boldsymbol{U}^{\mathrm{nom}}$, such that:
 
\begin{enumerate}
  \item \textbf{First-order expansion.}
        \begin{equation}
          \label{eq:first_order}
          \boldsymbol{U}^{\mathrm{AL}}(\beta)
            {=} \boldsymbol{U}^{\mathrm{nom}}
              {+} \beta\,\boldsymbol{H}^{-1}_{\mathrm{task}}
                \nabla J_{\mathrm{info}}(\boldsymbol{U}^{\mathrm{nom}})
              {+} O(\beta^2).
        \end{equation}
  \item \textbf{Information improvement.}
        \begin{align}
          \label{eq:info_improvement}
          J_{\mathrm{info}}\!\left(\boldsymbol{U}^{\mathrm{AL}}(\beta)\right)
            \geq J_{\mathrm{info}}\!\left(\boldsymbol{U}^{\mathrm{nom}}\right) \nonumber \\
            \text{for all sufficiently small } \beta > 0.
        \end{align}
\end{enumerate}
\end{theorem}
 
\begin{proof}
Since $\boldsymbol{U}^{\mathrm{nom}}$ lies in the interior of $\mathcal{F}$ by condition~(i), both $(P_0)$ and $(P_\beta)$ reduce locally to unconstrained optimization over $\mathcal{N}$. First-order optimality for $(P_\beta)$ requires
\begin{equation}
  \label{eq:foc}
  \nabla J_{\mathrm{task}}\!\left(\boldsymbol{U}^{\mathrm{AL}}(\beta)\right)
    - \beta\,\nabla J_{\mathrm{info}}\!\left(
        \boldsymbol{U}^{\mathrm{AL}}(\beta)\right)
  = \boldsymbol{0}.
\end{equation}
At $\beta = 0$, condition~\eqref{eq:foc} yields
$\nabla J_{\mathrm{task}}(\boldsymbol{U}^{\mathrm{nom}}) = \boldsymbol{0}$,
which is satisfied since $\boldsymbol{U}^{\mathrm{nom}}$ is a strict local
minimizer of $(P_0)$.
Define the residual map $\boldsymbol{F}(\boldsymbol{U},\beta)
    := \nabla J_{\mathrm{task}}(\boldsymbol{U})
       - \beta\,\nabla J_{\mathrm{info}}(\boldsymbol{U}),$
so that~\eqref{eq:foc} is equivalent to
$\boldsymbol{F}(\boldsymbol{U}^{\mathrm{AL}}(\beta),\beta) = \boldsymbol{0}$.
By condition~(ii), $\boldsymbol{F}$ is continuously differentiable on
$\mathcal{N}\times[0,\bar{\beta}]$, and by condition~(iii),
\begin{equation*}
  \frac{\partial\boldsymbol{F}}{\partial\boldsymbol{U}}
    \bigg|_{(\boldsymbol{U}^{\mathrm{nom}},\,0)}
  = \nabla^2 J_{\mathrm{task}}(\boldsymbol{U}^{\mathrm{nom}})
  = \boldsymbol{H}_{\mathrm{task}}
  \succ \boldsymbol{0}
\end{equation*}
is invertible.
By the Implicit Function Theorem, there exist $\bar{\beta} > 0$ and a
differentiable map $\beta\mapsto\boldsymbol{U}^{\mathrm{AL}}(\beta)$
satisfying $\boldsymbol{F}(\boldsymbol{U}^{\mathrm{AL}}(\beta),\beta) =
\boldsymbol{0}$ for all $\beta\in[0,\bar{\beta}]$, with $\boldsymbol{U}^{\mathrm{AL}}(0) = \boldsymbol{U}^{\mathrm{nom}}$.
 
\smallskip
\noindent\textit{Proof of~\eqref{eq:first_order}.}
Differentiating $\boldsymbol{F}(\boldsymbol{U}^{\mathrm{AL}}(\beta),\beta) =
\boldsymbol{0}$ with respect to $\beta$ and evaluating at
$(\boldsymbol{U}^{\mathrm{nom}}, 0)$ gives
\begin{equation}
  \label{eq:derivative}
  \frac{\mathrm{d}\boldsymbol{U}^{\mathrm{AL}}}{\mathrm{d}\beta}
  \bigg|_{\beta=0}
  = \boldsymbol{H}^{-1}_{\mathrm{task}}
    \nabla J_{\mathrm{info}}(\boldsymbol{U}^{\mathrm{nom}}).
\end{equation}
A first-order Taylor expansion of $\boldsymbol{U}^{\mathrm{AL}}(\beta)$
around $\beta = 0$ then gives~\eqref{eq:first_order}.
 
\smallskip
\noindent\textit{Proof of~\eqref{eq:info_improvement}.}
A first-order Taylor expansion of $J_{\mathrm{info}}$ around
$\boldsymbol{U}^{\mathrm{nom}}$ gives
\begin{align}
  J_{\mathrm{info}}\!\left(\boldsymbol{U}^{\mathrm{AL}}(\beta)\right)
    &= J_{\mathrm{info}}(\boldsymbol{U}^{\mathrm{nom}})
       + \nabla J_{\mathrm{info}}(\boldsymbol{U}^{\mathrm{nom}})^\top \nonumber \\
         &\bigl(\boldsymbol{U}^{\mathrm{AL}}(\beta)
               - \boldsymbol{U}^{\mathrm{nom}}\bigr)
       + O(\beta^2).
\end{align}
Substituting~\eqref{eq:first_order}:
\begin{align}
  &J_{\mathrm{info}}\!\left(\boldsymbol{U}^{\mathrm{AL}}(\beta)\right)
    = J_{\mathrm{info}}(\boldsymbol{U}^{\mathrm{nom}})
       + \nonumber \\ &\beta\,\nabla J_{\mathrm{info}}(\boldsymbol{U}^{\mathrm{nom}})^\top
         \boldsymbol{H}^{-1}_{\mathrm{task}}
         \nabla J_{\mathrm{info}}(\boldsymbol{U}^{\mathrm{nom}})
       + O(\beta^2).
\end{align}
Since $\boldsymbol{H}^{-1}_{\mathrm{task}} \succ \boldsymbol{0}$, the
quadratic form
$\nabla J_{\mathrm{info}}(\boldsymbol{U}^{\mathrm{nom}})^\top
  \boldsymbol{H}^{-1}_{\mathrm{task}}
  \nabla J_{\mathrm{info}}(\boldsymbol{U}^{\mathrm{nom}}) \geq 0$,
with equality only when
$\nabla J_{\mathrm{info}}(\boldsymbol{U}^{\mathrm{nom}}) = \boldsymbol{0}$,
i.e., when the nominal solution already maximizes information gain.
In the non-trivial case
$\nabla J_{\mathrm{info}}(\boldsymbol{U}^{\mathrm{nom}}) \neq \boldsymbol{0}$, the quadratic term is strictly positive, and $J_{\mathrm{info}}(\boldsymbol{U}^{\mathrm{AL}}(\beta)) >
  J_{\mathrm{info}}(\boldsymbol{U}^{\mathrm{nom}})$
for all sufficiently small $\beta > 0$,
establishing~\eqref{eq:info_improvement}.
\end{proof}
Theorem~\ref{thm:perturbation} thus confirms that the active-learning MPC strictly improves information gain over the nominal solution whenever the nominal trajectory is not already maximally informative, with the magnitude of improvement scaling with the misalignment between the task gradient and the information gradient.

Note that theorem~\ref{thm:perturbation} is a local result that relies on condition~(i): $\boldsymbol{U}^{\mathrm{nom}}$ must lie in the interior of $\mathcal{F}$, i.e., no constraints may be active at the nominal solution. When state, input, or CBF constraints are active, the first-order optimality condition acquires active-constraint multipliers, and a sensitivity analysis via the parametric KKT system is required.

\end{document}